\documentclass{article}

\usepackage{microtype}
\usepackage{graphicx}
\usepackage{subcaption}
\usepackage{comment}
\usepackage{booktabs} 
\usepackage{wrapfig}
\usepackage{hyperref}
\usepackage{xurl}

\usepackage[accepted]{icml2026}

\usepackage{amsmath}
\usepackage{amssymb}
\usepackage{mathtools}
\usepackage{amsthm}
\usepackage{multirow}

\usepackage[capitalize,noabbrev]{cleveref}

\theoremstyle{plain}
\newtheorem{theorem}{Theorem}[section]

\newtheorem{lemma}[theorem]{Lemma}
\newtheorem{corollary}[theorem]{Corollary}
\theoremstyle{definition}

\theoremstyle{remark}

\usepackage[disable,textsize=tiny]{todonotes}

\icmltitlerunning{Adaptive and Asymmetric Surrogate Gradients for Training Deep Spiking Neural Networks}

\begin{document}

\twocolumn[
  \icmltitle{A$^2$SG: Adaptive and Asymmetric Surrogate Gradients \\ 
  for Training Deep Spiking Neural Networks}



  \icmlsetsymbol{equal}{*}
  \icmlsetsymbol{intern}{\textdagger}

  \begin{icmlauthorlist}
    \icmlauthor{Yechan Kang}{equal,kist,ku}
    \icmlauthor{Yongjin Kweon}{equal,intern,lig}
    \icmlauthor{Mingyeong Seo}{kist,ku}
    \icmlauthor{Sohee Park}{kist,ku}
    \icmlauthor{Yeonguk jeon}{intern,sk}
    \icmlauthor{Jongkil Park}{kist}
    \icmlauthor{Hyun Jae Jang}{kist}
    \icmlauthor{Jaewook Kim}{kist}
    \icmlauthor{YeonJoo Jeong}{kist}
    \icmlauthor{Suyoun Lee}{kist}
    \icmlauthor{Seongsik Park}{kist}
  \end{icmlauthorlist}

  \icmlaffiliation{kist}{Center for Semiconductor Technology, Korea Institute of Science and Technology (KIST), Seoul, Republic of Korea}
  \icmlaffiliation{lig}{LIG Defense\&Aerospace, Seongnam, Republic of Korea} 
  \icmlaffiliation{sk}{SK hynix Inc., Icheon, Republic of Korea}
  \icmlaffiliation{ku}{Department of Computer Science and Engineering, Korea University, Seoul, Republic of Korea}

  \icmlcorrespondingauthor{Seongsik Park}{seong.sik.park@kist.re.kr}

  \icmlkeywords{Machine Learning, ICML}

  \vskip 0.3in
]



\printAffiliationsAndNotice{\icmlEqualContribution. \textdagger Work done while interning at KIST.}

\begin{abstract}
Training deep spiking neural networks (SNNs) remains challenging due to sharp loss landscapes and temporal inconsistency caused by surrogate gradients.
To address these challenges, we propose a unified framework: adaptive and asymmetric surrogate gradients \textit{A$^2$SG}.
The adaptive gradients adjust an effective window for spatio-temporal adaptation, reducing spatial gradient variation and maintaining directional consistency of gradients over time.
The asymmetric gradients reflect neuronal dynamics by assigning larger gradients to neurons with higher membrane potentials, and we prove that they yield lower variation than symmetric surrogates.
Our analysis further establishes a direct connection between local gradient variation and the curvature of the loss landscape, providing a principled explanation for how \textit{A$^2$SG} promotes convergence to flatter minima and improves generalization.
We conduct extensive experiments on diverse models, including CNN-based and Transformer-based SNNs, across various tasks such as image classification using both static and neuromorphic datasets, as well as segmentation.
The results demonstrate that \textit{A$^2$SG} consistently improves accuracy and energy efficiency, establishing it as a general and reliable solution for training deep SNNs.
Our code is available at \url{https://github.com/KIST-NCL/A2SG.git}.

\end{abstract}

\section{Introduction}

Spiking neural networks (SNNs) have emerged as energy-efficient next-generation neural networks that operate based on spikes~\citep{maass1997networks}.
In particular, by leveraging the low-power characteristics of SNNs and the superior learning capabilities of deep neural networks (DNNs), deep SNNs have shown the potential for energy-efficient artificial intelligence in various fields~\citep{park2020t2fsnn,kim2020spiking,kim2022beyond,yao2025scaling}. 
Recently, deep SNNs have been applied to applications, including image segmentation~\citep{kim2022beyond,lei2025spike2former}, object detection~\citep{kim2020spiking,su2023deep}, and language modeling~\citep{bal2024spikingbert,10.5555/3692070.3694321}, as well as to diverse model architectures, such as Transformers~\citep{zhou2023spikformer,yao2025scaling}.
These advancements have been largely driven by the adoption of gradient-based training with surrogate gradients~\citep{wu2018spatio,neftci2019surrogate}.

Despite their essential role in training deep SNNs, research on effective surrogate gradient functions remains limited, contributing to the performance gap between DNNs and deep SNNs.
Several studies have attempted to mitigate the mismatch between surrogate and true gradients by adaptively adjusting surrogate functions.
However, existing approaches have predominantly focused on gradient sparsity for adaptation~\citep{lian2023learnable,lin2023efficient} or impose substantial computational overhead~\citep{li2021differentiable}, restricting the training performance or hindering practical deployment.
Furthermore, few studies have designed surrogate functions that take into account the impact of surrogate gradients on generalization performance.

In this work, we introduce adaptive and asymmetric surrogate gradients (\textit{A$^2$SG}) to enhance the training of deep SNNs.
The adaptive component leverages spatio-temporal adaptation, dynamically adjusting the surrogate gradient window to suppress spatial fluctuations of gradients and align their directions across timesteps.
The asymmetric component allocates larger gradients to neurons with greater membrane potential, effectively prioritizing those closer to firing and promoting convergence to flatter minima.
These designs are motivated by our theoretical analysis, which shows that a larger variation in local gradient leads to sharper loss landscapes.
Moreover, we demonstrate that the proposed asymmetric surrogate gradient exhibits lower gradient variation compared to its symmetric counterparts.
In addition, we highlight temporal gradient confusion, which is caused by misaligned gradients across timesteps, as one of the obstacles for stable learning, motivating the need for spatio-temporal adaptation.
By combining these, we establish a unified strategy that stabilizes optimization, promotes convergence to flatter minima, and improves generalization.
We validate our proposed approaches through extensive experiments on both static and neuromorphic datasets, spanning convolutional neural networks (CNNs) and Transformer-based models.
Across all benchmarks, \textit{A$^2$SG} achieves consistent gains in accuracy and energy efficiency, highlighting its effectiveness as a general solution for reliable deep SNN training.

\section{Related Work}

\subsection{Training Deep Spiking Neural Networks}

Recent studies have introduced deep SNNs achieving both high performance and energy efficiency~\citep{tavanaei2019deep}.
In these architectures, leaky integrate-and-fire (LIF) neurons are widely used for their computational simplicity and biological plausibility (Eqs.~\ref{eq:vmem_func}-\ref{eq:reset_func}).
Deep SNNs have been applied to various tasks such as image classification~\citep{hu2021spiking,fang2021deep}, object detection~\citep{kim2020towards,kim2020spiking}, semantic segmentation~\citep{kim2022beyond,lei2025spike2former}, and Transformer-based models~\citep{zhou2023spikformer,yao2025scaling}.
Recent studies have adopted direct training based on spatio-temporal backpropagation (STBP) with surrogate gradients~\citep{wu2018spatio}, as given in Eqs.~\ref{eq:STBP} and~\ref{eq:temporal_grad}. 
This method enables efficient training with fewer timesteps, yet a performance gap remains compared to DNNs.
To mitigate this gap, several studies have focused on addressing the gradient mismatch problem arising from the adoption of surrogate gradients~\citep{li2021differentiable,lian2023learnable}.
However, studies on gradient consistency during training have remained relatively limited.
Especially, in STBP, parameter updates are obtained by aggregating gradient contributions from all timesteps, and inconsistent temporal gradients can generate conflicting signals, a phenomenon we term \textit{temporal gradient confusion}.
This inconsistency hinders stable optimization and degrades learning performance, highlighting the need for strategies that explicitly mitigate it.

\subsection{Surrogate Gradients in Spatio-Temporal Backpropagation (STBP)}

Surrogate gradients have been employed to address the non-differentiability of the spiking function (Eq.~\ref{eq:spike_func}) during error backpropagation.
Although the adoption of surrogate gradients has dramatically improved the performance of deep SNNs, their use remains limited by inconsistencies with the true gradients.
To improve the learning performance, several studies have focused on adjusting the distribution of the membrane potential during training~\citep{guo2022loss,guo2023rmp,guo2023membrane,zhao2025improving}.
Various regularization strategies have been employed, such as maximizing the information within the membrane potential~\citep{guo2022loss} or minimizing quantization errors induced by the spike function~\citep{guo2023rmp}.
In addition, batch normalization~\citep{guo2023membrane} and KL loss~\citep{zhao2025improving} were applied to mitigate the inter-batch and temporal discrepancies in the membrane potential distribution, respectively.
While these studies improved performance by adjusting the membrane potential distribution, they did not fundamentally address the performance degradation inherent to surrogate gradients, highlighting the essential need for advancements in surrogate gradient design.

\begin{figure*}[!t]
    \centering
    \includegraphics[width=0.95\textwidth]{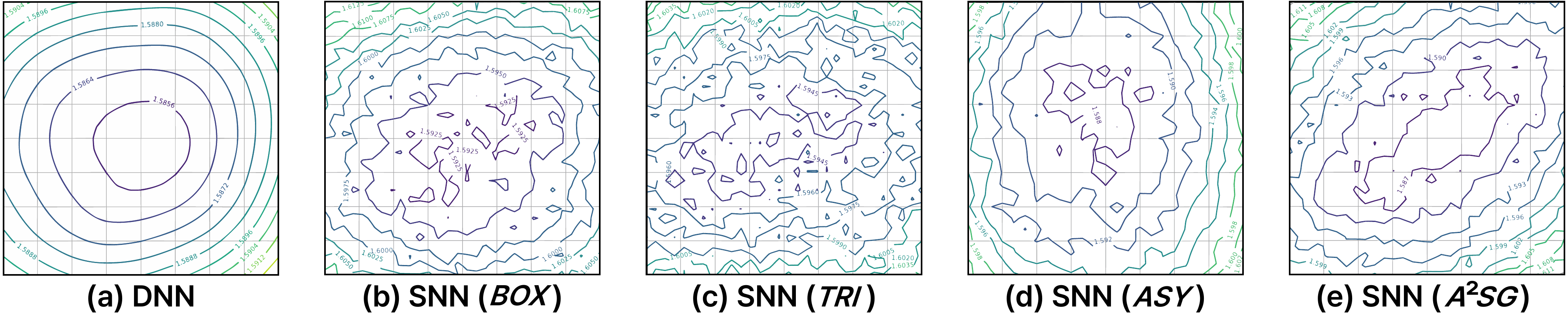}
    \caption{Visualization of the loss landscape in the Conv1 layer of VGG16 trained on CIFAR10. The results demonstrate that $A^2SG$ leads to a flatter region of the loss landscape compared to conventional surrogate gradients.
    }
    \label{fig:limitation_LS}
\end{figure*}

\subsection{Improving Surrogate Gradients Design for Deep SNNs}

Most surrogate gradients adopt static and symmetric function shapes to approximate the Dirac delta, which represents the derivative of the spiking function ($\frac{\partial s[t]}{\partial u[t]}$).
These functions preserve a constant area within the effective window $[V_{\textrm{th}}-\beta, V_{\textrm{th}}+\beta]$, with representative examples being the boxcar (\textit{BOX}) and triangle (\textit{TRI}) functions (Eqs.~\ref{ea:gradient_boxcar}-\ref{ea:gradient_trinangle}.)
To mitigate vanishing gradients, \citet{guo2024take} proposed directly delivering gradients to shallow layers, though gradient mismatch remains unresolved.
Beyond static functions, adaptive strategies have been explored to improve training further~\citep{li2021differentiable,lian2023learnable,lin2023efficient,ijcai2025p464}.
Dspike~\citep{li2021differentiable} employs finite-difference gradients to align surrogates with true gradients via cosine similarity, though its high computational cost limits scalability.
To reduce this overhead, other studies propose adjusting the effective window width to regulate gradient sparsity during training~\citep{lian2023learnable,lin2023efficient,ijcai2025p464}.
However, most of these approaches have focused on sparsity control to avoid gradient vanishing or explosion, relying on sparsity as an indirect indicator of learning quality.
Moreover, existing efforts have primarily concentrated on symmetric functions that mimic differential operators, while the impact of function shape itself remains underexplored.

\subsection{Flat Minima and Generalization}
Although gradient-based direct learning has significantly improved the performance, generalization remains a critical challenge for deep SNNs.
Numerous studies have reported that models converging to flat minima tend to achieve better generalization~\citep{hochreiter1997flat,keskar2017large,chaudhari2019entropy}.
Flatness is commonly assessed via the Hessian spectrum~\citep{ghorbani2019hessian}, but direct Hessian computations are intractable for large models.
As a practical alternative, the Fisher information matrix (FIM) (Eq.~\ref{eq:FIM}) is frequently used, since its eigenvalues are known to capture the curvature of the loss landscape, with smaller values indicating flatter minima~\citep{liao2018approximate,karakida2019universal,martens2020new,kim2022fishersam}.
In addition to measurement, several training procedures aim to encourage convergence to flat minima.
These include entropy-based biasing toward wide valleys~\citep{chaudhari2019entropy} and curvature-aware updates using FIM-based criteria~\citep{kim2022fishersam}.
Related studies have also linked sharp minima to large-batch training and poor generalization~\citep{keskar2017large}, further motivating flatness-oriented training strategies.
Overall, prior works suggest that guiding optimization toward flat minima is an effective approach for improving generalization.

\section{Sharp Loss Landscape from Surrogate Gradient Learning}
Deep SNNs trained with surrogate gradients tend to converge to sharper regions of the loss landscapes than DNNs. 
The flatness of the loss landscape is characterized by its curvature, quantified by the Hessian $\mathbf{H} = \nabla_{\mathbf{w}}^2 L(\mathbf{w})$, with respect to the parameter vector $\mathbf{w}$.
For clarity, we analyze the second derivative of the loss with respect to a single weight, which corresponds to a diagonal entry of the Hessian.
By the chain rule, the second derivative of the loss with respect to a weight $w$ can be described as 
\begin{equation} \label{eq:second_deriv}
\frac{\partial^2L}{\partial w^2} = \frac{\partial^2L}{\partial \phi^2}(\phi'(u)x)^2 + \frac{\partial L}{\partial \phi}\phi''(u)x^2 \textrm{,}
\end{equation}
where $x$ denotes the pre-synaptic activation and $u=wx$.
For DNNs with an activation function $\phi(u)$, assume the first and second derivatives are bounded, i.e., $|\phi'(u)| \leq c_1$ and $|\phi''(u)| \leq c_2$ for finite constants $c_1$ and $c_2$, which yields
\begin{equation} \label{eq:second_deriv_size_dnn}
\left| \frac{\partial^2 L}{\partial w^2} \right|_{DNN} = \mathcal{O}(x^2) \textrm{.}
\end{equation}

For deep SNNs, a surrogate gradient function $f(u)$ introduces an effective window of width $\beta$ to approximate the non-differentiable spike function. 
As shown in Sec.~\ref{app:symmetric_scaling}, any symmetric surrogate with fixed area 
satisfies
\begin{equation}
\begin{aligned}
\|H'\|_\infty=\|f\|_\infty=\Omega(\beta^{-1}) \quad
\textrm{and} \quad \\
\|H''\|_\infty=\|f'\|_\infty=\Omega(\beta^{-2}).
\end{aligned}
\end{equation}
Substituting them into the chain-rule expansion (Eq.~\ref{eq:second_deriv}) yields
\begin{equation} \label{eq:second_deriv_size_snn}
\left| \frac{\partial^2 L}{\partial w^2} \right|_{SNN} = \Omega(\frac{x^2}{\beta^2}) \textrm{.}
\end{equation}

Most prior works adopt a \emph{narrow effective window} ($\beta<1$) to better approximate the Dirac delta around the threshold, which empirically improves training stability and convergence~\citep{wu2018spatio, neftci2019surrogate}.
However, under this condition, Eq.~\ref{eq:second_deriv_size_snn} shows that the Hessian magnitude is amplified by a factor of $1/\beta^{2}$ in surrogate-trained SNNs, driving optimization toward sharper regions than in DNNs.
Note that the surrogate gradient determines where the model converges on the loss landscape, which is defined by the loss function and the model architecture.
In addition, the binary and temporally sparse nature of spikes concentrates gradients and increases their variation, further biasing convergence toward sharp regions, as shown in Fig.~\ref{fig:limitation_LS}.
Consistent with this analysis, \textit{TRI} converges to a sharper region of the loss landscape than \textit{BOX} because, under area normalization, its steeper slopes imply larger curvature in the gradient updates.
This is derived in Sec.~\ref{app:ls_land_tri_box_comp} and confirmed empirically through Fig.~\ref{fig:limitation_LS}-(b) and (c).

\begin{figure*}[t]
    \centering
    \includegraphics[width=0.95\textwidth]{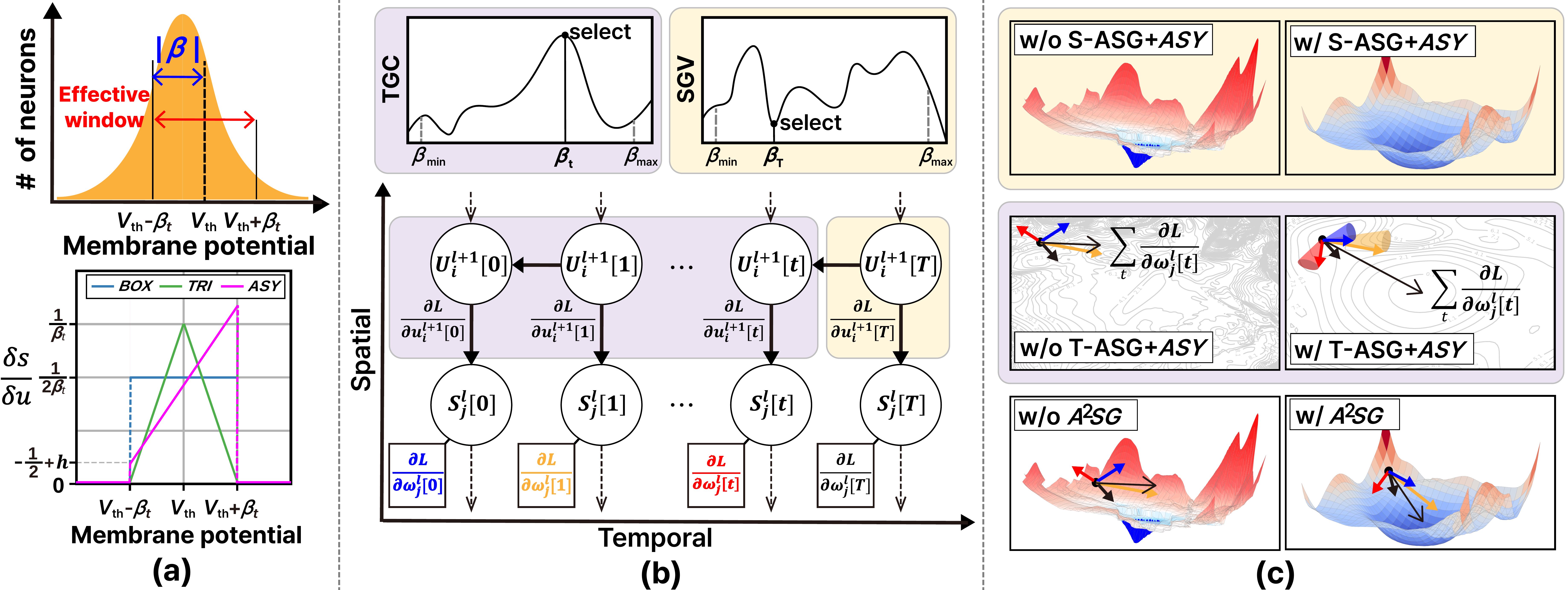}
    \caption{ Overview of the proposed $\textit{$A^2$SG}$ framework. (a) Effective window (red) is modulated by the parameter $\beta$ (blue); three shapes (\textit{BOX}, \textit{TRI}, and \textit{ASY}) are shown.
    (b) Temporal adaptive surrogate gradient (T-ASG) selects $\beta_t$ that maximizes temporal gradient consistency (TGC). While spatial adaptive surrogate gradient (S-ASG) selects $\beta_t$ that minimizes spatial gradient variation (SGV).
    (c) S-ASG+\textit{ASY} promotes flat minima, T-ASG+\textit{ASY} stabilizes gradient directions, and \textit{A$^2$SG} achieves robust convergence by combining both.
    }
    \label{fig:overview}
\end{figure*}

\section{\textit{A$^2$SG}: Adaptive and Asymmetric Surrogate Gradients}

We introduce \textit{A$^2$SG} to address two limitations in surrogate-based SNN training: sharp loss landscapes and temporal gradient confusion.
The adaptive component consists of two policies: spatial and temporal adaptations to address sharp loss landscapes and temporal gradient confusion, respectively.
The spatial adaptation adjusts the surrogate effective window to reduce the variance of the local gradient ($\frac{\partial L}{\partial u}$), encouraging convergence to flatter minima.
The temporal adaptation aligns local gradients per timestep to mitigate temporal inconsistency.
We demonstrate that the dispersion of local gradients affects the curvature of the loss landscape and that temporal gradients are formed by aggregating these local gradients.
Thus, adjusting $\beta$ during training effectively reduces spatial variability and enhances temporal alignment, improving generalization under non-stationary dynamics.

In addition, we analyze how the functional shape of the surrogate affects local gradient variation and introduce an asymmetric surrogate that considers neuronal dynamics.
By allocating larger gradients to neurons with greater accumulated membrane potential, the asymmetric form further suppresses the gradient variance and alleviates sharpness.
When utilized together, the spatio-temporal adaptation and the asymmetric components complement each other, effectively addressing the significant limitations of surrogate-gradient learning.
They stabilize the optimization process and yield flatter solutions with improved generalization.
The overall framework of \textit{A$^2$SG} is illustrated in Fig.~\ref{fig:overview}.

\subsection{Relation between Variation of Local Gradient and Flatness of Loss Landscape}


To analyze how the variability of local gradients affects the flatness of the loss landscape, we focus on the coefficient of variation (CV) of the local gradients. For notational simplicity, we consider a fully connected (FC) layer, but the principle naturally extends to other neural network layers and architectures.
Let $W \in \mathbb{R}^{m \times n}$ denote the weight matrix of the FC layer. The gradient with respect to $W$, vectorized, is given by:
$\textbf{g} = \operatorname{vec}\left( \frac{\partial L}{\partial W} \right) = \mathbf{a_{\mathrm{in}}} \otimes \boldsymbol{\delta} \textrm{,}$
where $\textbf{a}_{\mathrm{in}} \in \mathbb{R}^n$ is the input vector and $\delta \in \mathbb{R}^m$ is the backpropagated error vector.
The FIM is then defined as:
\begin{equation} \label{eq:FIM}
\textbf{F} = \mathbb{E}[\mathbf{g} \mathbf{g}^\top] = \mathbb{E} \left[ (\mathbf{a}_{\mathrm{in}} \otimes \boldsymbol{\delta})(\mathbf{a}_{\mathrm{in}} \otimes \boldsymbol{\delta})^\top \right] \textrm{.}
\end{equation}
For analytical clarity, the $\delta$ can be decomposed into its mean and a zero-mean fluctuation term:
$
\boldsymbol{\delta} = \mu \mathbf{1} + \boldsymbol{\epsilon}
$
, where the $\mu$ is the mean of $\delta$, $\mathbf{1}$ is the all-ones vector and $\epsilon$ denotes a small perturbation. Substituting this into the definition of $\mathbf{F}$, we obtain:
\begin{equation} \label{eq:FIM_delta}
\mathbf{F} = \mu^2 \mathbb{E} \left[ (\mathbf{a}_{\mathrm{in}} \otimes \mathbf{1})(\mathbf{a}_{\mathrm{in}} \otimes \mathbf{1})^\top \right] + \mathbf{R} = \mathbf{F_0} + \mathbf{R}\textrm{,}
\end{equation}
where $\mathbf{F_0}$ is a rank-1 matrix and the perturbation $\mathbf{R}$ is bounded as $||R||_2 \leq c\mu \mathrm{CV}(\delta)$ for some constant $c$.
This shows that as $\mathrm{CV}(\delta)$ becomes smaller, the FIM approaches the rank-1 matrix $F_0$, indicating that the loss landscape has a dominant curvature direction. By matrix perturbation theory ~\citep{greenbaum2020first}, the largest eigenvalue of the FIM is bounded as:
\begin{equation} \label{eq:Maximum_eigenvalue}
\begin{aligned}
\lambda_{\max}(F) \leq \mu^2 \lambda_{\max} \left( \mathbb{E} \left[ (a_{\mathrm{in}} \otimes \mathbf{1})(a_{\mathrm{in}} \otimes \mathbf{1})^\top \right] \right) \\+ c \mu^2 \mathrm{CV}(\delta)\textrm{,}
\end{aligned}
\end{equation}
where $\lambda_{\max}(\cdot)$ denotes the largest eigenvalue operator.  
Therefore, the largest eigenvalue grows linearly with $\mathrm{CV}(\delta)$, and reducing the CV of local gradients directly leads to a flatter loss landscape.

\begin{figure*}[t]
    \centering
    \includegraphics[width=0.95\textwidth]{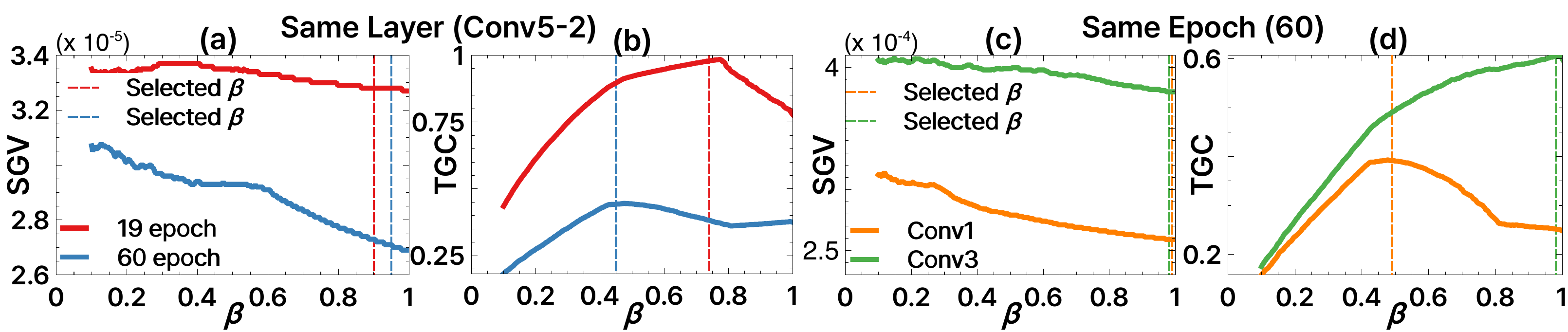}
    \caption{Graphs of SGV and TGC as a function of $\beta$. (a) and (b) represent SGV and TGC, at different epochs for the Conv5-2 layer, while (c) and (d) compare Conv1 and Conv3 at epoch 60. 
    Dashed lines indicate the selected $\beta$ through the proposed adaptive method in each training case.}
    \label{fig:beta_vs_sgv_tgc}
    \vspace{-0.5em}
\end{figure*}

\subsection{Spatio-Temporal Adaptive Surrogate Gradients (ST-ASG)}

As discussed in the previous section, reducing local gradient variability alleviates the sharpness of the loss landscape.
In addition, temporal gradient confusion can be mitigated by promoting alignment of local gradients across timesteps.
To achieve this, we introduce two metrics: spatial gradient variation (SGV) and temporal gradient consistency (TGC), which guide spatial and temporal adaptation, respectively.
SGV is defined as follows:
\begin{equation} \label{eq:SGV}
\operatorname{SGV}^{(l)}[T]: = \frac{\operatorname {Var}(\boldsymbol{\delta}^{l}[T])}{\operatorname {Mean}(|\boldsymbol{\delta}^{l}[T]|)} \textrm{,}
\end{equation}
where $\delta^{(l)}[T]$ denotes the backpropagated error in layer $l$ at the last timestep $T$.
To improve computational efficiency in practice, SGV employs the variance rather than the standard deviation in the denominator, differing from the conventional CV.
On the other hand, TGC at timestep $t$ is defined as the cosine similarity between the local gradients at adjacent timesteps:
\begin{equation} \label{eq:TGC}
\operatorname{TGC}^{(l)}[t]: = \operatorname{cos}(\boldsymbol{\delta}^{(l)}[t], \boldsymbol{\delta}^{(l)}[t+1]), t \in [1, T-1] \textrm{.}
\end{equation}
With these definitions, spatial adaptation is applied using SGV to suppress local gradient variation at the last timestep, where activations and gradients are relatively stable.
In this case, the adaptation objective is to minimize SGV.
Conversely, temporal adaptation is guided by TGC, which promotes alignment of local gradients between adjacent timesteps.
In this context, the goal of adaptation is to maximize TGC.
Since spatial adaptation is applied at the last timestep, it provides a stable reference direction for the temporal gradients.
Temporal adaptation, applied to preceding timesteps ($t<T$), then aligns the local gradients with this reference.
Spatio-temporal adaptation is achieved by anchoring the global gradient trajectory to the stable direction obtained at the last timestep while enforcing temporal consistency across earlier timesteps.

To realize this spatio-temporal adaptation in practice, we propose to adjust the width ($\beta$) of the effective window.
Our method is motivated by the observation that both SGV and TGC can be expressed as functions of $\beta$.
As illustrated in Fig.~\ref{fig:beta_vs_sgv_tgc}, these functions vary unpredictably with the training dynamics: the same layer exhibits different functional shapes across epochs, and even within a single epoch, distinct patterns may emerge across layers.
To robustly identify suitable values of $\beta$ under such variability, we employ a Bayesian search strategy. 
Implementation details of the adaptive method and the Bayesian optimization procedure are provided in Secs.~\ref{app:adaptive} and~\ref{app:beta_search}, respectively.

\subsection{Asymmetric Surrogate Gradients}
Symmetric functions, such as \textit{TRI} and \textit{BOX}, have been widely used as surrogate gradients in STBP, providing gradients based solely on the distance between a neuron’s membrane potential and the threshold.
However, they fail to account for neuronal dynamics such as integration and firing.
Thus, the relative magnitude of the accumulated membrane potential is not effectively reflected in the training process.
To solve this problem, we propose an asymmetric (\textit{ASY}) surrogate, defined as 
\begin{equation} \label{eq:asymmetric_sg_function}
\frac{\partial s}{\partial u} = f(u,\beta) = \frac{1}{2\beta}\cdot (u-V_\textrm{th})+h, 
\quad u\in[V_\textrm{th}-\beta, V_{\textrm{th}}+\beta] \textrm{,}
\end{equation}
where $h$ is a gradient-bias term added across the effective window that controls the overall magnitude of the surrogate gradient.
A smaller $h$ reduces the gradient magnitude, suppressing the gradients of low-potential neurons and yielding sparser gradients, whereas a larger $h$ raises the gradient across the window.
By assigning larger gradients to neurons with more strongly accumulated membrane potential, the asymmetric form reflects each neuron's dynamic behavior in the learning process.

This neuron's behavior-aware design also leads to a reduction in gradient variability.
By concentrating gradient values in regions where membrane potential is high (and spike likelihood is greater), the \textit{ASY} function avoids spreading gradients across irrelevant low-activity regions.
This focused gradient allocation reduces unnecessary variability, resulting in more stable training.

Theoretical results formalize this intuition as follows.
\begin{theorem}[CV-Minimizing Symmetric Function under Area and Boundary Constraints]~\label{th:cv_min_syn_func}
Let $f_{\mathrm{asy}}(u)$ and $f_{\mathrm{sym}}(u)$ be asymmetric and symmetric surrogate gradient functions defined over $[a, b]$, satisfying the boundary condition $f(a)=f(b)=0$, non-negativity $f(u)\geq0$, and area constraint $\int_a^b f(u)\,du = c$, where $a=\theta-\beta$ and $b = \theta + \beta$.
Suppose the membrane potential $u \sim \mathcal{N}(\mu, \sigma^2)$ with $\mu < a$, so that $p(u)$ is decreasing on $[a, b]$.
Then the unique function $f^*$ that minimizes the CV over all such admissible functions is the symmetric triangular function.
\end{theorem}
\begin{proof}
Please refer to Thm.~\ref{th:app_cv_min_syn_func}.
\end{proof}
\vspace{-1em}
Thm.~\ref{th:cv_min_syn_func} shows that, under area and boundary constraints, the symmetric function minimizing gradient CV is a triangular function.
This sets a lower bound of CV for symmetric functions under the given constraints.
Based on this fact, we verify that asymmetric has a lower CV than symmetric.
\begin{theorem}[CV Comparison of Asymmetric and Symmetric Surrogates]~\label{th:cv_comp}
Let $f_{\mathrm{asy}}(u)$ and $f_{\mathrm{sym}}(u)$ be asymmetric and symmetric surrogate gradient functions defined over $[a, b]$ under the same constraints as in Thm.~\ref{th:cv_min_syn_func}
Then, under a linear approximation of the Gaussian, where $L=b-a$ and $\kappa=a-\mu$, we have:
\[
\mathrm{CV}_{\mathrm{asy}} < \mathrm{CV}_{\mathrm{sym}} \quad \text{if } L \kappa> \sigma^2.
\]
\end{theorem}
\begin{proof}
Please refer to Thm.~\ref{th:app_cv_comp}.
\end{proof}
\vspace{-1em}
Thms.~\ref{th:cv_min_syn_func} and~\ref{th:cv_comp} demonstrate that the proposed asymmetric surrogate gradient, by incorporating membrane potential accumulation, achieves a lower CV than its symmetric counterparts.
These theoretical findings indicate that the asymmetric surrogate gradient not only better reflects neuronal dynamics but also promotes more stable and efficient learning, thereby facilitating convergence to flatter minima.
Experimental validations of Thms.~\ref{th:cv_min_syn_func} and~\ref{th:cv_comp} are provided in Fig.~\ref{fig:conv5-2_SGV_TGC_FIM_beta}-(a), where \textit{ASY} function consistently exhibits lower gradient variance than \textit{TRI} function.
Fig.~\ref{fig:theoretical_validation} presents \textit{L$\kappa$} and $\sigma^2$ of Conv1 and Conv5-2 during training on VGG16 with CIFAR10.
The graphs show that the condition $L\kappa > \sigma^2$ in Thm.~\ref{th:cv_comp} becomes satisfied across layers as training progresses.
This confirms that our assumption is supported by the experimental results.

\section{Experiments}

\setlength{\textfloatsep}{0pt}
\begin{table}[t]
    \centering
    \setlength{\tabcolsep}{1.5pt}
    \renewcommand{\arraystretch}{0.7}
    \caption{Comparison with other adaptive surrogate gradient methods (\textit{ARC}: Arctangent, ST-ASG: Our spatio-temporal adaptive surrogate gradient).}
    \label{tab:results_comp_adaptive}
    {
    \resizebox{\columnwidth}{!}{%
        \begin{tabular}{ccclcc}
            \toprule
             & \multirow{2}{*}{Architectures} & \multirow{2}{*}{Functions} & \multirow{2}{*}{Methods}  & Time & Acc.\\
             &&&&Steps & (\%) \\
            \midrule
            \multirow{6}{*}[-4pt]{\centering\rotatebox[origin=c]{90}{\textbf{CIFAR10}}}
            &ResNet18
             &\textit{ARC} &Dspike~\citep{li2021differentiable}  &4 & 93.66$\pm$0.05 \\\cmidrule{2-6}
             &\multirow{5}{*}{ ResNet19}
             & \textit{TRI} & CPNG~\citep{lin2023efficient}    &6 & 94.10$\pm$0.05 \\
            && \textit{TRI} & ST-ASG & 4 & 96.41$\pm$0.16 \\
            && \textit{BOX} & LSG~\citep{lian2023learnable}   &  4 & 95.17$\pm$0.05 \\
            & & \textit{BOX}& ST-ASG & 4 & 96.01$\pm$0.05 \\
            & &\textit{ASY} & \textit{\textbf{\boldmath{$A^2$}SG}} &  \textbf{4} & \textbf{96.74$\pm$0.05} \\

            \midrule
            \multirow{6}{*}[-4pt]{\centering\rotatebox[origin=c]{90}{\textbf{CIFAR100}}}
            & ResNet18
            & {\textit{ARC}} & Dspike~\citep{li2021differentiable}  & {4} & {73.35$\pm$0.14} \\
            \cmidrule{2-6}
            
            & \multirow{5}{*}{ ResNet19}
            & {\textit{TRI}}& CPNG~\citep{lin2023efficient}   & {6} &{75.37$\pm$0.05} \\
            && \textit{TRI}& ST-ASG  & 4 & 80.46$\pm$0.06 \\
            & & {\textit{BOX}}&LSG~\citep{lian2023learnable}    & {4} & {76.85$\pm$0.10} \\
            && \textit{BOX}& ST-ASG  & 4 & 78.60$\pm$0.16 \\
            && \textit{ASY} &\textit{\textbf{\boldmath{$A^2$}SG}}  & \textbf{4} & \textbf{81.05$\pm$0.05} \\
            
            \bottomrule
        \end{tabular}
    }}
\end{table}

We evaluated the effectiveness of the proposed method on various datasets, including static image datasets such as CIFAR10, CIFAR100~\citep{krizhevsky2009learning}, and ImageNet~\citep{deng2009imagenet} as well as a neuromorphic dataset such as CIFAR10-DVS~\citep{li2017cifar10}. 
We conducted experiments on both CNN and Transformer models. 
To further show the versatility of our method, we also evaluated it on the ADE20K~\citep{zhou2017scene} dataset for semantic segmentation.
For more details about the experimental setup, please refer to Secs.~\ref{app:computing} and ~\ref{app:experimental_setup}.

\begin{figure*}[!t]
    \centering
    \includegraphics[width=0.9\textwidth]{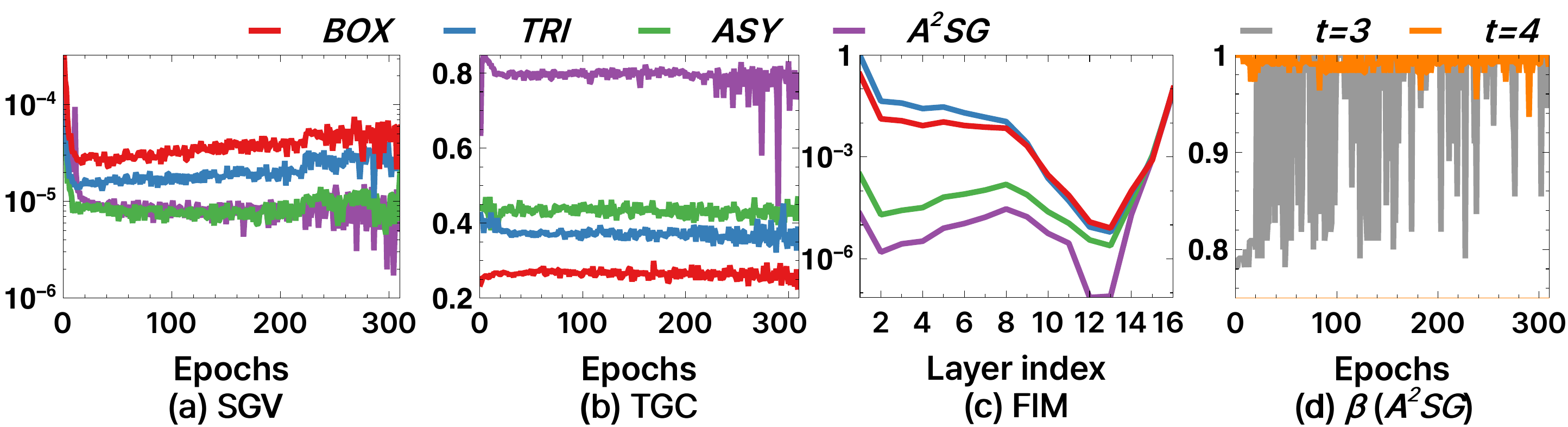}
    \caption{Comparison of SGV, TGC, FIM, and $\beta$ dynamics across different surrogate gradient functions at Conv5-2 layer. (a) SGV over epochs, (b) TGC over epochs, (c) FIM across layer indices, (d) $\beta$ dynamics at $t=3$ and at $t=4$ with $A^2SG$.}
    \label{fig:conv5-2_SGV_TGC_FIM_beta}
\end{figure*}

\subsection{Effect of \textit{A$^2$SG} on Gradient Dynamics and Feature Learning}
We analyze the effect of \textit{A$^2$SG} on gradient variation and temporal consistency.
As shown in Fig.~\ref{fig:conv5-2_SGV_TGC_FIM_beta}-(a) and (b), \textit{A$^2$SG} maintains low SGV and high TGC throughout training.
Furthermore, in Fig.~\ref{fig:conv5-2_SGV_TGC_FIM_beta}-(c), measurement of the maximum eigenvalue of the FIM for each layer reveals that \textit{A$^2$SG} consistently achieves the lowest eigenvalues throughout the network. 
These results corroborate our theoretical findings, demonstrating that lower CV directly leads to convergence toward flatter minima in the loss landscape.
In addition, Fig.~\ref{fig:conv5-2_SGV_TGC_FIM_beta}-(d) illustrates the adaptive selection of $\beta$ across training epochs at $t=3$ and $t=4$, where $\beta$ is chosen to minimize SGV and maximize TGC, respectively.
Additional analysis of the \textit{ASY} function and the results for the Conv1 and Conv3 layers are provided in Secs.~\ref{app:quant_anal} and ~\ref{app:other_layer}, respectively.

\setlength{\textfloatsep}{0pt}
\begin{table}[t]
    \centering
    \setlength{\tabcolsep}{2pt}
    \caption{Comparison with current state-of-the-art approaches on CIFAR10/100.}
    \label{tab:results_comp_direct}
    \renewcommand{\arraystretch}{0.6}
    \resizebox{\columnwidth}{!}{%
        {\normalsize
        \begin{tabular}{cclcc}
            \toprule
             & \multirow{2}{*}{Architectures} & \multirow{2}{*}{Methods} & Time & Acc.\\
             &&&Steps &(\%) \\
            \midrule
            \multirow{14}{*}[-18pt]{\centering\rotatebox[origin=c]{90}{\textbf{CIFAR10}}}
            &\multirow{4}{*}[-3pt]{VGG16}
               & IM~\citep{guo2022loss} & 5 & 93.85 \\
              && RMP~\citep{guo2023rmp} & 4 & 93.33 \\
              && MPBN~\citep{guo2023membrane} & 4 & 94.44 \\
              && \textit{\textbf{\boldmath{$A^2$}SG}} & \textbf{4} & \textbf{95.29$\pm$0.05} \\ \cmidrule{2-5}

            &\multirow{10}{*}[-9pt]{ResNet19}
               & IM~\citep{guo2022loss} & 4 & 95.40 \\
             && RMP~\citep{guo2023rmp} & 4 & 95.51 \\
             && TET~\citep{deng2022temporal} & 4 & 94.44 \\
             && TAB~\citep{jiang2024tab} & 4 & 94.76 \\
             && ShortcutBP~\citep{guo2024take} & 2 & 95.36 \\
             \cmidrule{3-5}
             && \multirow{3}{*} {MPD-AGL~\citep{ijcai2025p464}} & 2 & 96.18 \\
             && & 4 & 96.35 \\ 
             && & 6 & 96.54 \\
             \cmidrule{3-5}
             && \multirow{2}{*}{\textit{\textbf{\boldmath{$A^2$}SG}}} & \textbf{2} & \textbf{96.34$\pm$0.02} \\ 
             &&  & \textbf{4} & \textbf{96.74$\pm$0.05} \\ \midrule

            \multirow{13}{*}[-20pt]{\centering\rotatebox[origin=c]{90}{\textbf{CIFAR100}}}
            &\multirow{4}{*}[-3pt]{VGG16}
               & IM~\citep{guo2022loss} & 5 & 70.18 \\
              && RMP~\citep{guo2023rmp} & 4 &  72.55\\
              && MPBN~\citep{guo2023membrane} & 4 & 74.74 \\
              && \textit{\textbf{\boldmath{$A^2$}SG}} & \textbf{4} & \textbf{75.21$\pm$0.08} \\ \cmidrule{2-5}
            &\multirow{9}{*}[-9pt]{ResNet19}
              & RMP~\citep{guo2023rmp} & 4 & 78.28 \\
             && TET~\citep{deng2022temporal} & 4 & 74.47 \\
             && TAB~\citep{jiang2024tab} & 4 & 76.81 \\
             && ShortcutBP~\citep{guo2024take} & 2 & 77.79\\
             \cmidrule{3-5}
             && \multirow{3}{*} {MPD-AGL~\citep{ijcai2025p464}} & 2 & 78.84 \\
             &&  & 4 & 79.72 \\ 
             &&  & 6 & 80.49 \\
             \cmidrule{3-5}
             && \multirow{2}{*}{\textit{\textbf{\boldmath{$A^2$}SG}}} & \textbf{2} & \textbf{79.18$\pm$0.01} \\ 
             && & \textbf{4} & \textbf{81.05$\pm$0.05} \\
            \bottomrule
        \end{tabular}
        }
    }   
\end{table}

To further examine how improved gradient dynamics translate into feature representations, we perform a t-SNE visualization of the learned feature representations (Fig.~\ref{fig:tSNE_base_ours}).
Models trained with \textit{A$^2$SG} exhibit class-separable features at early timestep $t=1$ and by $t=4$ the features of each class consolidate into compact and well-separated clusters.
These properties indicate that gradients are effectively propagated to deeper layers and maintain coherent directions across timesteps, facilitating convergence to flat minima and improving training ability.
In contrast, \textit{BOX} exhibits significant class overlap and less clear separation.
This demonstrates that \textit{A$^2$SG} enhances both training stability and generalization performance, as visualized by more robust and discriminative feature representations.
\begin{figure}[t]
    \centering
    \includegraphics[width=0.8\linewidth]{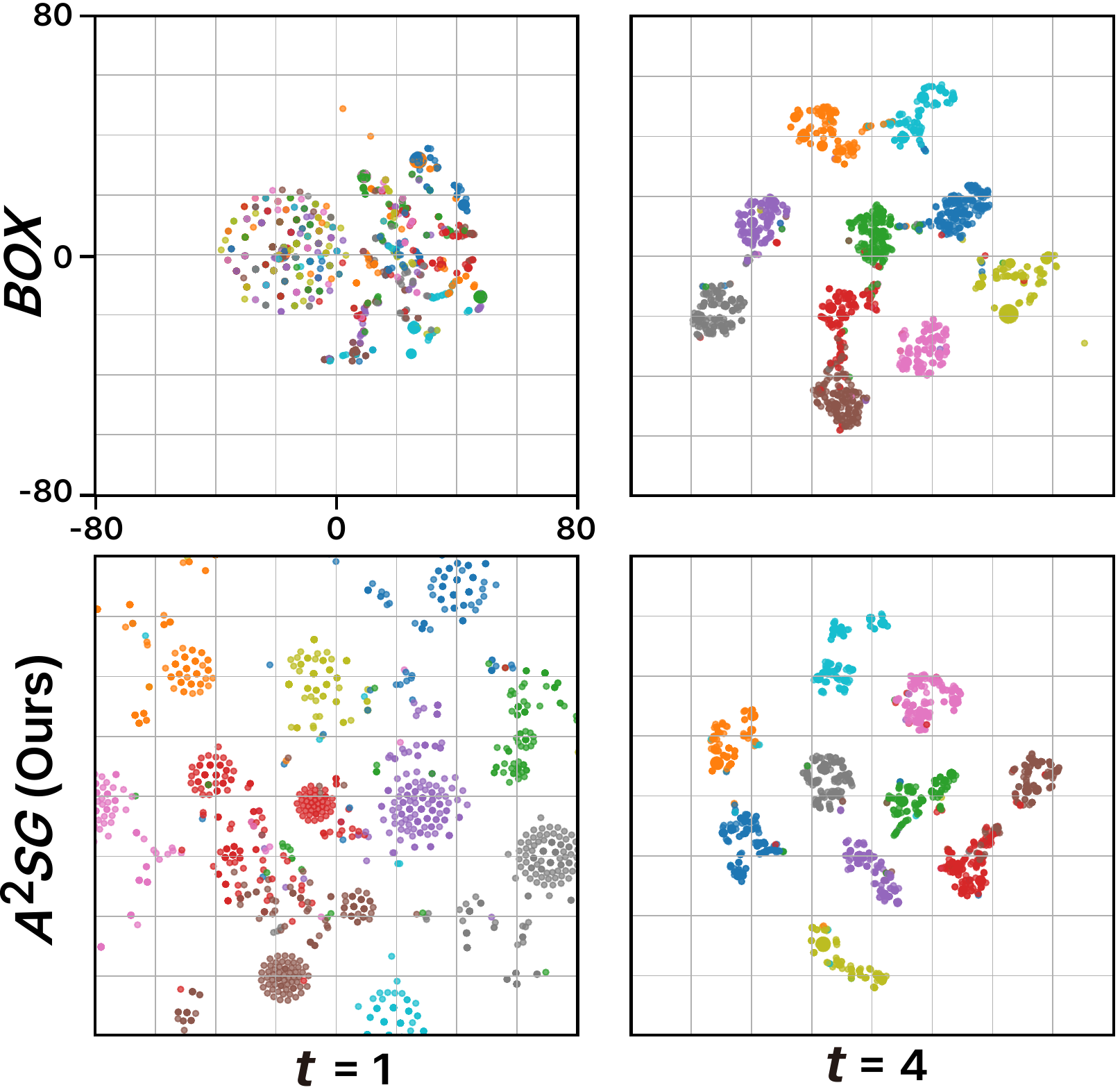}
    \caption{Comparison of t-SNE for VGG16 on CIFAR10 using \textit{BOX} (top) and \textit{A$^2$SG} (bottom) at $t=1$ and $4$.}
    \label{fig:tSNE_base_ours}
\end{figure}
\subsection{Comparison with Other Adaptive Surrogate Gradients}\label{Sec:comp_other_sg}
Tab.~\ref{tab:results_comp_adaptive} compares the accuracy of various adaptive surrogate gradient methods on CIFAR10 and CIFAR100 with ResNet18 and ResNet19.
When applied to \textit{BOX} and \textit{TRI}, ST-ASG consistently outperforms previous adaptive surrogate gradient methods on both datasets.
This superiority stems from its spatio-temporal adaptation: the spatial policy stabilizes local gradients to favor flatter minima, while the temporal policy aligns gradient directions across timesteps for more robust global updates.
Furthermore, \textit{$A^2$SG} achieves the highest accuracy among all methods.
On ResNet19, our approach outperforms LSG by $1.57\%$  on CIFAR10 (96.74\% vs. 95.17\%) and by $4.20\%$ on CIFAR100 (81.05\% vs. 76.85\%).
These results demonstrate the strong generalization capability of our method over existing adaptive surrogate gradient approaches.

\subsection{Comparison with State-of-the-Art Methods} \label{Sec:comp_sota}
As reported in Tabs.~\ref{tab:results_comp_direct} and \ref{tab:Comparison_with_other_Transformer}, our method outperforms prior approaches on CIFAR10 and CIFAR100 in accuracy and achieves both higher accuracy and lower power consumption on ImageNet.
In the case of E-SpikeFormer (Tab.~\ref{tab:Comparison_with_other_Transformer}), it adopted integer LIF neurons with multiple thresholds, which is distinct from the conventional SNNs with LIF neurons.
To efficiently train on large-scale datasets, it uses integer values as a substitute for the temporal spike trains.
Thus, this mechanism implicitly incorporates temporal dynamics, even when $T=1$.
The model can be regarded as operating with an implicit time step equal to the maximum integer spike count $D$, rather than a single time step.
Based on this perspective, we aligned the time step of the conventional LIF neurons with the integer activation value of I-LIF.
For example, when $T\times D$ is $1 \times 4$ for I-LIF, we applied S-ASG to neurons with an activation value of four, corresponding to the last time step.
We then sequentially applied T-ASG to neurons with activation values of three, two, and one.
For the effective window ($\beta$), we set it to be centered on each threshold, as in LIF with a single threshold.
For example, if $\textrm{th}_i$ is a threshold at integer $i$, the effective window of $\textrm{th}_i$ is set to $[\textrm{th}_i-\beta_i, \textrm{th}_i+\beta_i]$. 
From this state, we changed $\beta_i$ through our adaptive method.

\begin{table}[t]
    \centering
    \setlength{\tabcolsep}{2pt}
    \caption{Comparison with Transformer-based SNNs on ImageNet.
    Following E-SpikeFormer~\cite{yao2025scaling}, timesteps are denoted as \( T \times D \), where \( T \) is the number of timesteps and \( D \) indicates the upper bound of integer activations.
    }
    \label{tab:Comparison_with_other_Transformer}
    \renewcommand{\arraystretch}{0.6}
    \resizebox{\columnwidth}{!}{%
        \begin{tabular}{llccc}
            \toprule
            \multirow{2}{*}{Methods} & Param & Power & Time & Acc.
                                  \\ & (M)   & (mJ)  & Steps &(\%)\\ \midrule
            Spikformer~\citep{zhou2023spikformer}& 66.3 & 21.5 &$4 \times 1$ & 74.8 \\
             \midrule
            Meta-SpikeFormer~\citep{yao2024spikedriven}& 31.3 & 32.8 &$4 \times 1$ & 77.2\\
             \midrule
            E-SpikeFormer~\citep{yao2025scaling}& 10.0 &3.0 &$1 \times 4$ & 78.5\\
            E-SpikeFormer~\citep{yao2025scaling}& 173.0 &35.6 &$1 \times 4$ & 84.7\\
             \midrule
            \textbf{E-SpikeFormer + \textit{\boldmath{$A^2$}SG}} & \textbf{10.0} & \textbf{2.78} & $1 \times 4$ & \textbf{78.61$\pm$0.01} \\
            \textbf{E-SpikeFormer + \textit{\boldmath{$A^2$}SG}} & \textbf{173.0} & \textbf{35.64} & $1 \times 4$ & \textbf{85.43} \\
            \bottomrule
        \end{tabular}%
    }
\end{table}
Tab.~\ref{tab:results_comp_dvs} shows that our method maintains higher accuracy on CIFAR10-DVS with only four timesteps.
In Tab.~\ref{tab:result_segmentation}, our approach also attains higher mIoU and reduced power consumption on ADE20K, demonstrating its effectiveness for segmentation as well as its energy efficiency.
Moreover, Fig.~\ref{fig:segmentation_fig} illustrates qualitative improvements in segmentation when applied to E-SpikeFormer.
Notably, \textit{A$^2$SG} converges rapidly, achieving competitive accuracy with only two timesteps, comparable to other state-of-the-art methods that require four timesteps.
Overall, these results demonstrate that by improving the design of surrogate gradients, our method provides a general and principled strategy for enhancing the training performance of deep SNNs across diverse architectures and tasks.

\setlength{\textfloatsep}{0pt}
\begin{table}[t]
    \centering
    \setlength{\tabcolsep}{2pt}
    \caption{Comparisons with other works on CIFAR10-DVS (* denotes our implementation).}
    \label{tab:results_comp_dvs}
    \renewcommand{\arraystretch}{0.6}
    \resizebox{\columnwidth}{!}{%
        \begin{tabular}{clcc}
            \toprule
            \multirow{2}{*}{Architectures} & \multirow{2}{*}{Methods} & Time & Acc.\\
            &&Steps & (\%) \\
            \midrule
             \multirow{4}{*}[-3pt]{VGGSNN}
            &STBP-tdBN~\citep{zheng2021going}* & 4 & 81.30$\pm$1.00 \\
            &HSD~\citep{zhong2024towards} & 5 & 81.10 \\
            &TMC~\citep{yan2025training} & 4 & 81.76 \\
            &\textit{\textbf{\boldmath{$A^2$}SG}} & \textbf{4} & \textbf{82.36$\pm$0.01} \\
            \bottomrule
        \end{tabular}
    }
\end{table}

\setlength{\textfloatsep}{0pt}
\begin{table}[!t]
    \centering
    \caption{Performance of segmentation on ADE20K. These methods use the pre-trained models on ImageNet as the backbone, then add segmentation heads for fine-tuning.}
    \label{tab:result_segmentation}
        \setlength{\tabcolsep}{2pt}
    \renewcommand{\arraystretch}{0.6}
    \resizebox{\columnwidth}{!}{%
        {\scriptsize
        \begin{tabular}{llccc}
            \toprule
            \multirow{2}{*}{Methods} & Param & Power & Time & mIoU.
                                  \\ & (M)   & (mJ)  & Steps & (\%)\\ \midrule
            Meta-SpikeFormer~\citep{yao2024spikedriven}&16.5 & 88.1 &$4 \times 1$ & 33.6\\
            E-SpikeFormer~\citep{yao2025scaling}& 11.0 & 27.2 &$1 \times 4$ & 40.1
              \\ \midrule
            \textbf{E-SpikeFormer + \textit{\boldmath{$A^2$}SG}} & 11.0 & \textbf{25.2} & $1 \times 4$ & \textbf{40.94} \\
            \bottomrule
        \end{tabular}%
        }
    }
\end{table}

\begin{table}[t]
  \centering
  \setlength{\tabcolsep}{10pt}
  \footnotesize
  \caption{Comparison with Sharpness-Aware Minimization (SAM) on CIFAR10 with VGG16 ($T=4$, perturbation radius $\rho=0.05$).}
  \label{tab:results_comp_sam}
      \begin{tabular}{lccc}
          \toprule
          \multirow{2}{*}{Methods} & Acc.  & \# of  & Latency\\
          &(\%)&Spikes (k) &(sec/epoch)\\
          \midrule
          \textit{BOX}                                   & 94.84 & 94.6 &74\\
          \textit{BOX} + SAM  & 94.87 & 92.5 &140\\
          \textit{BOX} + ST-ASG                          & 94.98 & 93.6 &85\\
           \textit{\textbf{\boldmath{$A^2$}SG}}           & \textbf{95.29} & \textbf{84.9} &\textbf{85}\\
          \bottomrule
      \end{tabular}%
\end{table}

\begin{figure}[t]
    \centering
    \includegraphics[width=1\linewidth]{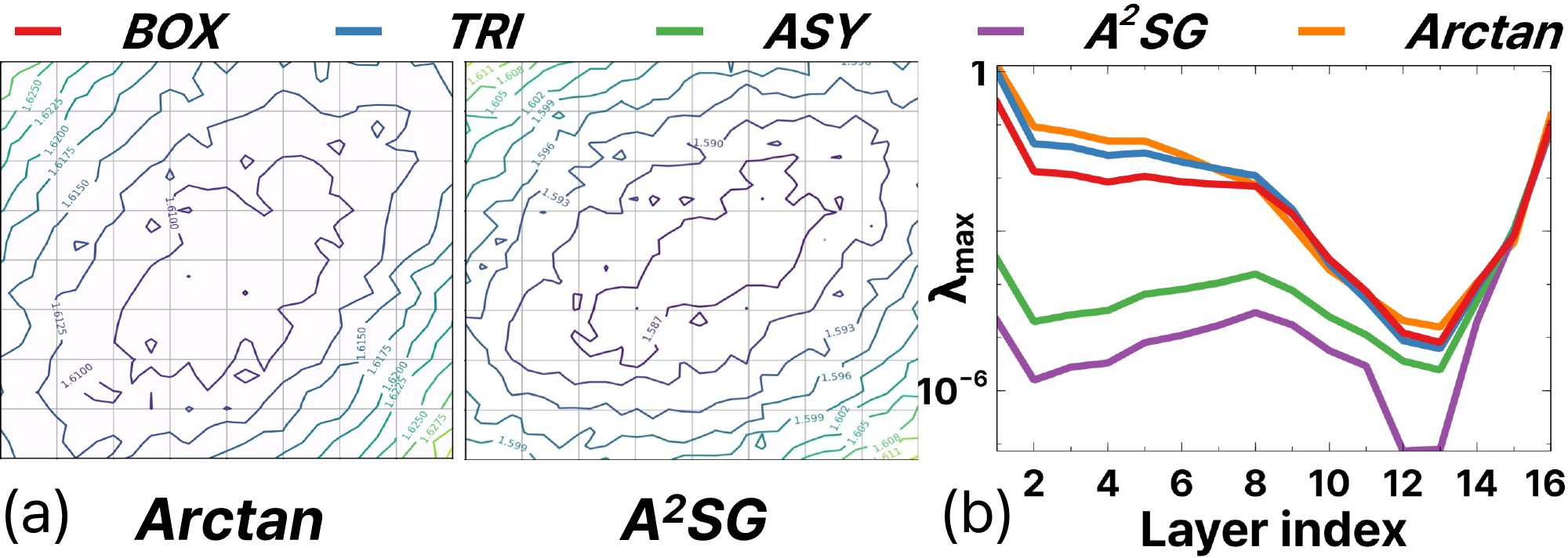}
    \caption{(a) Loss landscape of \textit{Arctan} and $A^2SG$. (b) Layer-wise $\lambda_{\max}$ of the FIM for \textit{BOX}, \textit{TRI}, \textit{ASY}, \textit{Arctan}, and $A^2SG$.}
    \label{fig:smooth_gradient}
\end{figure}

\subsection{Comparison with Flat Minima Approaches} \label{sec:smooth_gradient}

$A^2SG$ promotes convergence to flat minima through both its asymmetric surrogate design and spatio-temporal adaptation.
We compare it against two alternative routes to flat minima: Sharpness-Aware Minimization (SAM)~\citep{foret2021sharpness} and a smooth surrogate (i.e., Arctan~\citep{fang2021incorporating}).
SAM was originally proposed for DNNs and, to the best of our knowledge, has not been applied to deep SNNs; we thus implement it for SNNs to enable this comparison.
As shown in Tab.~\ref{tab:results_comp_sam}, applying SAM to the \textit{BOX} baseline yields only a marginal accuracy gain while producing more spikes and doubling the training cost.
This inefficiency arises because SAM does not consider the temporal gradient inconsistency of SNNs, exploring the loss landscape via an additional gradient computation at every step.
These factors together make SAM difficult to apply effectively to deep SNNs.
In contrast, $A^2SG$ attains higher accuracy and lower spike counts at only $\sim$15\% overhead, reaching flat minima more efficiently than adding a separate parameter-space optimizer such as SAM.

We next compare $A^2SG$ with the smooth \textit{Arctan} surrogate $\sigma'(x) = \frac{1}{1+(\pi x)^2}$~\citep{fang2021incorporating} on VGG16/CIFAR10 to examine whether smoothness directly contributes to flat-minima convergence.
As shown in Fig.~\ref{fig:smooth_gradient}-(a), \textit{Arctan} converges to a less irregular region than BOX and TRI, reflecting the effect of its smoothness.
The curvature of this region, however, remains comparable to that of \textit{BOX} and \textit{TRI} and higher than that of $A^2SG$.
This distinction is quantified in Fig.~\ref{fig:smooth_gradient}-(b): the $\lambda_{\max}$ of the FIM for \textit{Arctan} is comparable to those of the non-smooth symmetric surrogates \textit{BOX} and \textit{TRI}, whereas \textit{ASY} and $A^2SG$ attain substantially lower values.
These results indicate that smoothness reduces the irregularity of the converged region but not its curvature; the asymmetric design, rather than smoothness, is the decisive factor for flat-minima convergence.


%
\subsection{Ablation Studies}\label{sec:ablation}

Tab.~\ref{tab:results_ablation} summarizes ablation results, highlighting the contributions of the adaptive and asymmetric components.
Incorporating the spatial adaptive surrogate gradient (S-ASG) improves accuracy and reduces spike count, while the temporal adaptive surrogate gradient (T-ASG) alone slightly improves accuracy with a slight increase in spikes.
Their combination (ST-ASG) further stabilizes training, and adding the asymmetric surrogate (\textit{$A^2$SG}) yields the highest accuracy (95.29\%) with the lowest spike count, confirming the benefit of integrating all components.
We further estimate the computational overhead of the proposed method by measuring the wall-clock time of one training epoch.
As reported in Tab.~\ref{tab:latency_scaling}, the relative overhead of \textit{$A^2$SG} consistently decreases with $T$, from $14.86\%$ at $T=4$ to $11.11\%$ at $T=12$. 
Since \textit{$A^2$SG} modifies only the backward pass while the forward pass grows linearly with $T$, the Bayesian optimization cost is amortized over a larger forward computation, keeping \textit{$A^2$SG} efficient at longer timesteps.

\begin{table}[t]
\centering
\centering
\caption{Ablation study on CIFAR10/100 with VGG16, comparing spatial (S-ASG), temporal (T-ASG), spatio-temporal adaptation (ST-ASG), and ST-ASG with \textit{ASY} (\textit{A$^2$SG}).}
\label{tab:results_ablation}
    {\scriptsize  
    \renewcommand{\arraystretch}{0.5}
        \setlength{\tabcolsep}{8pt}
        \begin{tabular}{clccc}
            \toprule
             &\multirow{2}{*}{Methods} & Acc. & \# of Spikes & Latency \\ 
             &&  (\%)& ($\times10^3$) & (sec/epoch)\\
             \midrule
            \multirow{5}{*}[-5pt]{\centering\rotatebox[origin=c]{90}{\textbf{CIFAR10}}}&\textit{BOX} (Baseline) & 94.84$\pm$0.05 & 94.6$\pm$1.0 & 74 (+0\%) \\ 
            \cmidrule{2-5}
            &w/ S-ASG & 94.97$\pm$0.04 & 80.0$\pm$0.8 &82 (+11\%)\\
            &w/ T-ASG  & 94.94$\pm$0.05 & 98.7$\pm$1.8 &83 (+12\%)\\
            &w/ ST-ASG & 94.98$\pm$0.03 & 93.6$\pm$2.3 &85 (+15\%)\\
            \cmidrule{2-5}
            &\textbf{\textit{\boldmath{$A^2$}SG} (Ours)} & \textbf{95.29$\pm$0.04} & \textbf{84.9$\pm$1.8} &85 (+15\%)\\
            \midrule
             \multirow{5}{*}[-4pt]{\centering\rotatebox[origin=c]{90}{\textbf{CIFAR100}}}&\textit{BOX} (Baseline) & 74.24$\pm$0.08 & 100.3$\pm$1.2 &  74 (+0\%) \\ 
            \cmidrule{2-5}
            &w/ S-ASG & 74.57$\pm$0.08 & 93.2$\pm$0.7 &  82 (+11\%)\\
            &w/ T-ASG  & 74.54$\pm$0.06 & 104.9$\pm$2.2 & 83 (+12\%)\\
            &w/ ST-ASG & 74.73$\pm$0.07 & 100.2$\pm$0.2 & 85 (+15\%)\\
            \cmidrule{2-5}
            &\textbf{\textit{\boldmath{$A^2$}SG} (Ours)} & \textbf{75.21$\pm$0.08} & \textbf{100.2$\pm$0.4} &85 (+15\%)\\
            \bottomrule
        \end{tabular}%
    }
\end{table}

\begin{table}[!t]
\centering
\setlength{\tabcolsep}{4pt}
\caption{Per-epoch latency on CIFAR10 with VGG16 across different timesteps.}
\label{tab:latency_scaling}
\renewcommand{\arraystretch}{0.5}
\footnotesize
\begin{tabular}{cccc}
\toprule
Timesteps & Baseline (s/epoch) & \textit{A$^2$SG} (s/epoch) & Overhead \\
\midrule
4  & 74  & 85  & +14.86\% \\
6  & 109 & 125 & +14.68\% \\
8  & 131 & 147 & +12.21\% \\
10 & 155 & 173 & +11.61\% \\
12 & 180 & 200 & +11.11\% \\
\bottomrule
\end{tabular}
\end{table}

\subsection{Sensitivity Analysis}
In this section, we conduct a sensitivity analysis of the hyperparameters of $A^{2}SG$, covering two aspects: the Bayesian search and the gradient-bias parameter $h$ of the \textit{ASY} surrogate gradient in Eq.~\ref{eq:asymmetric_sg_function}.
We first analyze the Bayesian search hyperparameters.
The default configuration uses a $\beta$ update frequency of 1 epoch, Bayesian optimization parameters $(n_{\text{obs}}, n_{\text{eval}}) = (100, 150)$, and a search radius $\delta = 0.05$.
As shown in Tab.~\ref{tab:hyperparameter}, varying the update frequency, $(n_{\text{obs}}, n_{\text{eval}})$, and $\delta$ around their defaults yields accuracy differences within 0.25\%, indicating that $A^2SG$ is robust to these hyperparameter variations.

\begin{table}[t]
\centering
\footnotesize
\centering
\caption{Comparison with different $\beta$ update frequencies (epoch) and Bayesian optimization hyperparameter ($n_{obs}$, $n_{eval}$, $\delta$) settings on CIFAR10 with VGG16.}
\label{tab:hyperparameter}
    \renewcommand{\arraystretch}{0.4}
    \setlength{\tabcolsep}{12pt}
        \begin{tabular}{ccccc}
            \toprule
             Epoch &$n_{obs}$ &$n_{eval}$&$\delta$ &Acc. (\%) \\ 
             \midrule
    
            1&100&150&0.05& 95.29$\pm$0.04  \\
            \cmidrule{1-5}
            
            50&\multirow{2}{*}{100}&\multirow{2}{*}{150}&\multirow{2}{*}{0.05}& 95.10$\pm$0.02 \\
            100&&&& 95.06$\pm$0.06  \\
            \cmidrule{1-5}
    
            \multirow{2}{*}{1}&10&15&\multirow{2}{*}{0.05}& 95.10$\pm$0.06  \\
            &300&450&& 95.31$\pm$0.02  \\
            \cmidrule{1-5}
    
            \multirow{2}{*}{1}&\multirow{2}{*}{100}&\multirow{2}{*}{150}& 0.01&95.15$\pm$0.08  \\
            &&&0.10& 95.24$\pm$0.03  \\
            \bottomrule
        \end{tabular}%
\end{table}

\begin{table}[!t]
\centering
\setlength{\tabcolsep}{2pt}
\caption{Sensitivity to $h$ of the \textit{ASY} surrogate gradient on CIFAR10 with VGG16.}
\label{tab:h_sensitivity}
\footnotesize
\begin{tabular}{lcccccccc}
\toprule
$h$ & 0.1 & 0.3 & 0.5 & \textbf{0.6} & 0.7 & 1.0 & 2.0 & 3.0 \\
\midrule
Acc. (\%) & 20.30 & 49.43 & 95.21 & \textbf{95.29} & 94.93 & 94.54 & 93.06 & 88.02 \\
\bottomrule
\end{tabular}
\end{table}

We further analyze the sensitivity of the gradient-bias term $h$ in the ASY surrogate on VGG16/CIFAR10 (Tab.~\ref{tab:h_sensitivity}).
For small $h$ (e.g., $h < 0.5$), the gradient bias is insufficient, making the overall gradient magnitude too small.
In addition, part of the effective window takes negative values, and the resulting positive and negative gradients partially cancel out when aggregated, corrupting the learning signal and causing a sharp drop in accuracy.
For $0.5 \le h \le 1.0$, the gradient remains consistent across the window and accuracy is stable, with the best result at $h = 0.6$.
For large $h$ (e.g., $h > 1.0$), the excessive gradient magnitude gradually degrades accuracy.
These results confirm that $A^2SG$ is robust to $h$ within a moderate range.

\subsection{Compatibility}
To validate the compatibility of our approach, we additionally applied $A^2SG$ to other neuron models and learning methods.
Specifically, we evaluated it with PLIF neurons~\citep{fang2021incorporating} or RMP-loss~\citep{guo2023rmp}, and the corresponding results are reported in Tabs.~\ref{tab:other_neuron} and ~\ref{tab:other_methods}, respectively.
The experimental results show improvements in both accuracy and efficiency, which consistently demonstrate that our method is compatible with various neuron models and training methods.

\subsection{Noise Robustness}
In this section, we analyze the noise robustness of \textit{A$^2$SG}.
We experimented with deletion noise, removing a proportion of spikes from each layer, and report the results in Tab.~\ref{tab:noise_robust}.
Compared to the \textit{BOX}, \textit{A$^2$SG} achieves higher accuracy under the deletion noise, while also exhibiting a smaller reduction in total spike counts.
Fig.~\ref{fig:noise} illustrates the weight distributions, where \textit{A$^2$SG} shows lower variance than \textit{BOX}.
This distribution indicates reduced reliance on specific weights, thereby improving generalization and enhancing robustness to errors~\citep{tsai2021formalizing}.

\section{Conclusion}

In this work, we proposed \textit{A$^2$SG}, a unified framework for training deep SNNs. By integrating spatio-temporal adaptation with a neuron-aware asymmetric design, \textit{A$^2$SG} reduces gradient variability, stabilizes optimization, and encourages convergence to flatter minima.
Our theoretical analysis establishes the link between gradient variation and loss landscape curvature.
Moreover, we prove that the asymmetric surrogate achieves lower variation than its symmetric counterparts.
Extensive experiments across diverse SNN architectures and tasks demonstrate that \textit{A$^2$SG} consistently improves accuracy, robustness, and efficiency, highlighting surrogate gradient design as a key factor for reliable and scalable SNN training.
We note that formal convergence proofs for surrogate gradients, including a rigorous error bound for the \textit{ASY} surrogate as a function of $h$, remain an open problem. 

\newpage
\section*{Acknowledgements}
This work was supported in part by the National Research Foundation of Korea (NRF) (No. NRF-2021R1C1C2010454) and the Institute of Information \& Communications Technology Planning \& Evaluation (IITP) (No. RS-2025-02218733) grants funded by the Korea government (MSIT), and the Korea Institute of Science and Technology (KIST) institutional program (No. 26E0020).

\section*{Impact Statement}
This paper presents work whose goal is to advance the training of spiking neural networks.
There are many potential societal consequences of our work, none of which we feel must be specifically highlighted here.

\bibliography{a2sg_paper}

\begin{thebibliography}{50}
\providecommand{\natexlab}[1]{#1}
\providecommand{\url}[1]{\texttt{#1}}
\expandafter\ifx\csname urlstyle\endcsname\relax
  \providecommand{\doi}[1]{doi: #1}\else
  \providecommand{\doi}{doi: \begingroup \urlstyle{rm}\Url}\fi

\bibitem[Bal \& Sengupta(2024)Bal and Sengupta]{bal2024spikingbert}
Bal, M. and Sengupta, A.
\newblock Spikingbert: Distilling bert to train spiking language models using implicit differentiation.
\newblock In \emph{Proceedings of the AAAI Conference on Artificial Intelligence}, volume~38, pp.\  10998--11006, 2024.

\bibitem[Chaudhari et~al.(2019)Chaudhari, Choromanska, Soatto, LeCun, Baldassi, Borgs, Chayes, Sagun, and Zecchina]{chaudhari2019entropy}
Chaudhari, P., Choromanska, A., Soatto, S., LeCun, Y., Baldassi, C., Borgs, C., Chayes, J., Sagun, L., and Zecchina, R.
\newblock Entropy-sgd: Biasing gradient descent into wide valleys.
\newblock \emph{Journal of Statistical Mechanics: Theory and Experiment}, 2019\penalty0 (12):\penalty0 124018, 2019.

\bibitem[Cubuk et~al.(2020)Cubuk, Zoph, Shlens, and Le]{cubuk2020randaugment}
Cubuk, E.~D., Zoph, B., Shlens, J., and Le, Q.~V.
\newblock Randaugment: Practical automated data augmentation with a reduced search space.
\newblock In \emph{Proceedings of the IEEE/CVF conference on computer vision and pattern recognition workshops}, pp.\  702--703, 2020.

\bibitem[Deng et~al.(2009)Deng, Dong, Socher, Li, Li, and Fei-Fei]{deng2009imagenet}
Deng, J., Dong, W., Socher, R., Li, L.-J., Li, K., and Fei-Fei, L.
\newblock Imagenet: A large-scale hierarchical image database.
\newblock In \emph{2009 IEEE conference on computer vision and pattern recognition}, pp.\  248--255. IEEE, 2009.

\bibitem[Deng et~al.(2022)Deng, Li, Zhang, and Gu]{deng2022temporal}
Deng, S., Li, Y., Zhang, S., and Gu, S.
\newblock Temporal efficient training of spiking neural network via gradient re-weighting.
\newblock In \emph{International Conference on Learning Representations}, 2022.

\bibitem[Fang et~al.(2021{\natexlab{a}})Fang, Yu, Chen, Huang, Masquelier, and Tian]{fang2021deep}
Fang, W., Yu, Z., Chen, Y., Huang, T., Masquelier, T., and Tian, Y.
\newblock Deep residual learning in spiking neural networks.
\newblock \emph{Advances in Neural Information Processing Systems}, 34:\penalty0 21056--21069, 2021{\natexlab{a}}.

\bibitem[Fang et~al.(2021{\natexlab{b}})Fang, Yu, Chen, Masquelier, Huang, and Tian]{fang2021incorporating}
Fang, W., Yu, Z., Chen, Y., Masquelier, T., Huang, T., and Tian, Y.
\newblock Incorporating learnable membrane time constant to enhance learning of spiking neural networks.
\newblock In \emph{Proceedings of the IEEE/CVF international conference on computer vision}, pp.\  2661--2671, 2021{\natexlab{b}}.

\bibitem[Foret et~al.(2021)Foret, Kleiner, Mobahi, and Neyshabur]{foret2021sharpness}
Foret, P., Kleiner, A., Mobahi, H., and Neyshabur, B.
\newblock Sharpness-aware minimization for efficiently improving generalization.
\newblock In \emph{International Conference on Learning Representations}, 2021.

\bibitem[Ghorbani et~al.(2019)Ghorbani, Krishnan, and Xiao]{ghorbani2019hessian}
Ghorbani, B., Krishnan, S., and Xiao, Y.
\newblock An investigation into neural net optimization via hessian eigenvalue density.
\newblock In \emph{International Conference on Machine Learning}, volume~97, pp.\  2232--2241. PMLR, 2019.

\bibitem[Greenbaum et~al.(2020)Greenbaum, Li, and Overton]{greenbaum2020first}
Greenbaum, A., Li, R.-c., and Overton, M.~L.
\newblock First-order perturbation theory for eigenvalues and eigenvectors.
\newblock \emph{SIAM review}, 62\penalty0 (2):\penalty0 463--482, 2020.

\bibitem[Guo et~al.(2022)Guo, Chen, Zhang, Liu, Wang, Huang, and Ma]{guo2022loss}
Guo, Y., Chen, Y., Zhang, L., Liu, X., Wang, Y., Huang, X., and Ma, Z.
\newblock Im-loss: information maximization loss for spiking neural networks.
\newblock \emph{Advances in Neural Information Processing Systems}, 35:\penalty0 156--166, 2022.

\bibitem[Guo et~al.(2023{\natexlab{a}})Guo, Liu, Chen, Zhang, Peng, Zhang, Huang, and Ma]{guo2023rmp}
Guo, Y., Liu, X., Chen, Y., Zhang, L., Peng, W., Zhang, Y., Huang, X., and Ma, Z.
\newblock Rmp-loss: Regularizing membrane potential distribution for spiking neural networks.
\newblock In \emph{Proceedings of the IEEE/CVF International Conference on Computer Vision}, pp.\  17391--17401, 2023{\natexlab{a}}.

\bibitem[Guo et~al.(2023{\natexlab{b}})Guo, Zhang, Chen, Peng, Liu, Zhang, Huang, and Ma]{guo2023membrane}
Guo, Y., Zhang, Y., Chen, Y., Peng, W., Liu, X., Zhang, L., Huang, X., and Ma, Z.
\newblock Membrane potential batch normalization for spiking neural networks.
\newblock In \emph{Proceedings of the IEEE/CVF International Conference on Computer Vision}, pp.\  19420--19430, 2023{\natexlab{b}}.

\bibitem[Guo et~al.(2024)Guo, Chen, Hao, Peng, Jie, Zhang, Liu, and Ma]{guo2024take}
Guo, Y., Chen, Y., Hao, Z., Peng, W., Jie, Z., Zhang, Y., Liu, X., and Ma, Z.
\newblock Take a shortcut back: Mitigating the gradient vanishing for training spiking neural networks.
\newblock \emph{Advances in Neural Information Processing Systems}, 37:\penalty0 24849--24867, 2024.

\bibitem[Hochreiter \& Schmidhuber(1997)Hochreiter and Schmidhuber]{hochreiter1997flat}
Hochreiter, S. and Schmidhuber, J.
\newblock Flat minima.
\newblock \emph{Neural Computation}, 9\penalty0 (1):\penalty0 1--42, 1997.

\bibitem[Hu et~al.(2021)Hu, Tang, and Pan]{hu2021spiking}
Hu, Y., Tang, H., and Pan, G.
\newblock Spiking deep residual networks.
\newblock \emph{IEEE Transactions on Neural Networks and Learning Systems}, 34\penalty0 (8):\penalty0 5200--5205, 2021.

\bibitem[Jiang et~al.(2024)Jiang, Zoonekynd, De~Masi, Gu, and Xiong]{jiang2024tab}
Jiang, H., Zoonekynd, V., De~Masi, G., Gu, B., and Xiong, H.
\newblock Tab: Temporal accumulated batch normalization in spiking neural networks.
\newblock In \emph{The Twelfth International Conference on Learning Representations}, 2024.

\bibitem[Jiang et~al.(2025)Jiang, Wang, Jiang, Fan, and Yan]{ijcai2025p464}
Jiang, J., Wang, L., Jiang, R., Fan, J., and Yan, R.
\newblock Adaptive gradient learning for spiking neural networks by exploiting membrane potential dynamics.
\newblock In \emph{Proceedings of the International Joint Conference on Artificial Intelligence}, pp.\  4164--4172. International Joint Conferences on Artificial Intelligence Organization, 8 2025.

\bibitem[Karakida et~al.(2019)Karakida, Akaho, and Amari]{karakida2019universal}
Karakida, R., Akaho, S., and Amari, S.-i.
\newblock Universal statistics of fisher information in deep neural networks.
\newblock \emph{Neural Computation}, 2019.

\bibitem[Keskar et~al.(2017)Keskar, Mudigere, Nocedal, Smelyanskiy, and Tang]{keskar2017large}
Keskar, N.~S., Mudigere, D., Nocedal, J., Smelyanskiy, M., and Tang, P. T.~P.
\newblock On large-batch training for deep learning: Generalization gap and sharp minima.
\newblock In \emph{International Conference on Learning Representations}, 2017.

\bibitem[Kim et~al.(2022{\natexlab{a}})Kim, Li, Hu, and Hospedales]{kim2022fishersam}
Kim, M., Li, D., Hu, S.~X., and Hospedales, T.
\newblock Fisher sam: Information geometry and sharpness aware minimisation.
\newblock In \emph{International Conference on Machine Learning}, pp.\  11148--11161. PMLR, 2022{\natexlab{a}}.

\bibitem[Kim et~al.(2020{\natexlab{a}})Kim, Park, Na, Kim, and Yoon]{kim2020towards}
Kim, S., Park, S., Na, B., Kim, J., and Yoon, S.
\newblock Towards fast and accurate object detection in bio-inspired spiking neural networks through bayesian optimization.
\newblock \emph{IEEE Access}, 9:\penalty0 2633--2643, 2020{\natexlab{a}}.

\bibitem[Kim et~al.(2020{\natexlab{b}})Kim, Park, Na, and Yoon]{kim2020spiking}
Kim, S., Park, S., Na, B., and Yoon, S.
\newblock Spiking-yolo: spiking neural network for energy-efficient object detection.
\newblock \emph{Proceedings of the AAAI Conference on Artificial Intelligence}, 34\penalty0 (07):\penalty0 11270--11277, 2020{\natexlab{b}}.

\bibitem[Kim et~al.(2022{\natexlab{b}})Kim, Chough, and Panda]{kim2022beyond}
Kim, Y., Chough, J., and Panda, P.
\newblock Beyond classification: Directly training spiking neural networks for semantic segmentation.
\newblock \emph{Neuromorphic Computing and Engineering}, 2\penalty0 (4):\penalty0 044015, 2022{\natexlab{b}}.

\bibitem[Krizhevsky et~al.(2009)Krizhevsky, Hinton, et~al.]{krizhevsky2009learning}
Krizhevsky, A., Hinton, G., et~al.
\newblock Learning multiple layers of features from tiny images.
\newblock 2009.

\bibitem[Lei et~al.(2025)Lei, Yao, Hu, Luo, Lu, Xu, and Li]{lei2025spike2former}
Lei, Z., Yao, M., Hu, J., Luo, X., Lu, Y., Xu, B., and Li, G.
\newblock Spike2former: Efficient spiking transformer for high-performance image segmentation.
\newblock In \emph{Proceedings of the AAAI Conference on Artificial Intelligence}, volume~39, pp.\  1364--1372, 2025.

\bibitem[Li et~al.(2017)Li, Liu, Ji, Li, and Shi]{li2017cifar10}
Li, H., Liu, H., Ji, X., Li, G., and Shi, L.
\newblock Cifar10-dvs: an event-stream dataset for object classification.
\newblock \emph{Frontiers in Neuroscience}, 11:\penalty0 244131, 2017.

\bibitem[Li et~al.(2021)Li, Guo, Zhang, Deng, Hai, and Gu]{li2021differentiable}
Li, Y., Guo, Y., Zhang, S., Deng, S., Hai, Y., and Gu, S.
\newblock Differentiable spike: Rethinking gradient-descent for training spiking neural networks.
\newblock \emph{Advances in Neural Information Processing Systems}, 34:\penalty0 23426--23439, 2021.

\bibitem[Lian et~al.(2023)Lian, Shen, Liu, Wang, Yan, and Tang]{lian2023learnable}
Lian, S., Shen, J., Liu, Q., Wang, Z., Yan, R., and Tang, H.
\newblock Learnable surrogate gradient for direct training spiking neural networks.
\newblock In \emph{Proceedings of the International Joint Conference on Artificial Intelligence}, pp.\  3002--3010, 2023.

\bibitem[Liao et~al.(2018)Liao, Drummond, Reid, and Carneiro]{liao2018approximate}
Liao, Z., Drummond, T., Reid, I., and Carneiro, G.
\newblock Approximate fisher information matrix to characterize the training of deep neural networks.
\newblock \emph{IEEE transactions on pattern analysis and machine intelligence}, 42\penalty0 (1):\penalty0 15--26, 2018.

\bibitem[Lin et~al.(2023)Lin, Deng, and Gu]{lin2023efficient}
Lin, H., Deng, S., and Gu, S.
\newblock Efficient surrogate gradients for training spiking neural networks.
\newblock In \emph{ICML 2023 Workshop on Differentiable Almost Everything: Differentiable Relaxations, Algorithms, Operators, and Simulators}, 2023.

\bibitem[Maass(1997)]{maass1997networks}
Maass, W.
\newblock Networks of spiking neurons: the third generation of neural network models.
\newblock \emph{Neural Networks}, 10\penalty0 (9):\penalty0 1659--1671, 1997.

\bibitem[Martens(2020)]{martens2020new}
Martens, J.
\newblock New insights and perspectives on the natural gradient method.
\newblock \emph{Journal of Machine Learning Research}, 21\penalty0 (146):\penalty0 1--76, 2020.

\bibitem[Neftci et~al.(2019)Neftci, Mostafa, and Zenke]{neftci2019surrogate}
Neftci, E.~O., Mostafa, H., and Zenke, F.
\newblock Surrogate gradient learning in spiking neural networks: Bringing the power of gradient-based optimization to spiking neural networks.
\newblock \emph{IEEE Signal Processing Magazine}, 36\penalty0 (6):\penalty0 51--63, 2019.

\bibitem[Park et~al.(2020)Park, Kim, Na, and Yoon]{park2020t2fsnn}
Park, S., Kim, S., Na, B., and Yoon, S.
\newblock T2fsnn: deep spiking neural networks with time-to-first-spike coding.
\newblock In \emph{2020 57th ACM/IEEE design automation conference (DAC)}, pp.\  1--6. IEEE, 2020.

\bibitem[Su et~al.(2023)Su, Chou, Hu, Li, Mei, Zhang, and Li]{su2023deep}
Su, Q., Chou, Y., Hu, Y., Li, J., Mei, S., Zhang, Z., and Li, G.
\newblock Deep directly-trained spiking neural networks for object detection.
\newblock In \emph{Proceedings of the IEEE/CVF International Conference on Computer Vision}, pp.\  6555--6565, 2023.

\bibitem[Tavanaei et~al.(2019)Tavanaei, Ghodrati, Kheradpisheh, Masquelier, and Maida]{tavanaei2019deep}
Tavanaei, A., Ghodrati, M., Kheradpisheh, S.~R., Masquelier, T., and Maida, A.
\newblock Deep learning in spiking neural networks.
\newblock \emph{Neural Networks}, 111:\penalty0 47--63, 2019.

\bibitem[Tsai et~al.(2021)Tsai, Hsu, Yu, and Chen]{tsai2021formalizing}
Tsai, Y.-L., Hsu, C.-Y., Yu, C.-M., and Chen, P.-Y.
\newblock Formalizing generalization and adversarial robustness of neural networks to weight perturbations.
\newblock \emph{Advances in Neural Information Processing Systems}, 34:\penalty0 19692--19704, 2021.

\bibitem[Wu et~al.(2018)Wu, Deng, Li, and Shi]{wu2018spatio}
Wu, Y., Deng, L., Li, G., and Shi, L.
\newblock Spatio-temporal backpropagation for training high-performance spiking neural networks.
\newblock \emph{Frontiers in Neuroscience}, 12:\penalty0 323875, 2018.

\bibitem[Xing et~al.(2024)Xing, Zhang, Ni, Xiao, Ju, Fan, Wang, Zhang, and Li]{10.5555/3692070.3694321}
Xing, X., Zhang, Z., Ni, Z., Xiao, S., Ju, Y., Fan, S., Wang, Y., Zhang, J., and Li, G.
\newblock Spikelm: towards general spike-driven language modeling via elastic bi-spiking mechanisms.
\newblock In \emph{International Conference on Machine Learning}, ICML'24. JMLR.org, 2024.

\bibitem[Yan et~al.(2025)Yan, Wang, Ma, Tang, Zheng, and Pan]{yan2025training}
Yan, J., Wang, C., Ma, D., Tang, H., Zheng, Q., and Pan, G.
\newblock Training high performance spiking neural network by temporal model calibration.
\newblock In \emph{International Conference on Machine Learning}, 2025.

\bibitem[Yao et~al.(2024)Yao, Hu, Hu, Xu, Zhou, Tian, XU, and Li]{yao2024spikedriven}
Yao, M., Hu, J., Hu, T., Xu, Y., Zhou, Z., Tian, Y., XU, B., and Li, G.
\newblock Spike-driven transformer v2: Meta spiking neural network architecture inspiring the design of next-generation neuromorphic chips.
\newblock In \emph{International Conference on Learning Representations}, 2024.

\bibitem[Yao et~al.(2025)Yao, Qiu, Hu, Hu, Chou, Tian, Liao, Leng, Xu, and Li]{yao2025scaling}
Yao, M., Qiu, X., Hu, T., Hu, J., Chou, Y., Tian, K., Liao, J., Leng, L., Xu, B., and Li, G.
\newblock Scaling spike-driven transformer with efficient spike firing approximation training.
\newblock \emph{IEEE Transactions on Pattern Analysis and Machine Intelligence}, 47\penalty0 (4):\penalty0 2973--2990, 2025.

\bibitem[Yun et~al.(2019)Yun, Han, Oh, Chun, Choe, and Yoo]{yun2019cutmix}
Yun, S., Han, D., Oh, S.~J., Chun, S., Choe, J., and Yoo, Y.
\newblock Cutmix: Regularization strategy to train strong classifiers with localizable features.
\newblock In \emph{Proceedings of the IEEE/CVF International Conference on Computer Vision}, pp.\  6023--6032, 2019.

\bibitem[Zhang et~al.(2018)Zhang, Cisse, Dauphin, and Lopez-Paz]{zhang2018mixup}
Zhang, H., Cisse, M., Dauphin, Y.~N., and Lopez-Paz, D.
\newblock mixup: Beyond empirical risk minimization.
\newblock \emph{International Conference on Learning Representations}, 2018.

\bibitem[Zhao et~al.(2025)Zhao, Shen, Dong, Li, and Zeng]{zhao2025improving}
Zhao, D., Shen, G., Dong, Y., Li, Y., and Zeng, Y.
\newblock Improving stability and performance of spiking neural networks through enhancing temporal consistency.
\newblock \emph{Pattern Recognition}, 159:\penalty0 111094, 2025.

\bibitem[Zheng et~al.(2021)Zheng, Wu, Deng, Hu, and Li]{zheng2021going}
Zheng, H., Wu, Y., Deng, L., Hu, Y., and Li, G.
\newblock Going deeper with directly-trained larger spiking neural networks.
\newblock In \emph{Proceedings of the AAAI Conference on Artificial Intelligence}, volume~35, pp.\  11062--11070, 2021.

\bibitem[Zhong et~al.(2024)Zhong, Hu, Liu, Huang, Ding, Yu, and Huang]{zhong2024towards}
Zhong, X., Hu, S., Liu, W., Huang, W., Ding, J., Yu, Z., and Huang, T.
\newblock Towards low-latency event-based visual recognition with hybrid step-wise distillation spiking neural networks.
\newblock In \emph{Proceedings of the 32nd ACM international conference on multimedia}, pp.\  9828--9836, 2024.

\bibitem[Zhou et~al.(2017)Zhou, Zhao, Puig, Fidler, Barriuso, and Torralba]{zhou2017scene}
Zhou, B., Zhao, H., Puig, X., Fidler, S., Barriuso, A., and Torralba, A.
\newblock Scene parsing through ade20k dataset.
\newblock In \emph{Proceedings of the IEEE conference on computer vision and pattern recognition}, pp.\  633--641, 2017.

\bibitem[Zhou et~al.(2023)Zhou, Zhu, He, Wang, YAN, Tian, and Yuan]{zhou2023spikformer}
Zhou, Z., Zhu, Y., He, C., Wang, Y., YAN, S., Tian, Y., and Yuan, L.
\newblock Spikformer: When spiking neural network meets transformer.
\newblock In \emph{International Conference on Learning Representations}, 2023.

\end{thebibliography}
\bibliographystyle{icml2026}

\newpage
\appendix
\onecolumn
\renewcommand{\thesection}{A}
\renewcommand{\thesubsection}{A.\arabic{subsection}}
\setcounter{section}{1}
\setcounter{subsection}{0}

\renewcommand{\thetable}{A\arabic{table}}
\setcounter{table}{0}
\renewcommand{\theequation}{A\arabic{equation}}
\setcounter{equation}{0}
\renewcommand{\thefigure}{A\arabic{figure}}
\setcounter{figure}{0}
\renewcommand{\thealgorithm}{A\arabic{algorithm}}
\setcounter{algorithm}{0}
\renewcommand{\thetheorem}{A\arabic{theorem}}
\setcounter{theorem}{0}
\renewcommand{\thecorollary}{A\arabic{corollary}}
\subsection{Leaky Integrate-and-fire (LIF) Neuron} \label{app:LIF}
The dynamics of a LIF neuron can be formulated as follows.
First, the membrane potential is updated at each timestep according to:
\begin{equation} \label{eq:vmem_func}
u_{i}^{l}[t] = \frac{1}{\tau} \left( v_{i}^{l}[t-1] + \sum\nolimits_{j} w_{ij}^{l}s_{j}^{l-1}[t] \right) \textrm{,}
\end{equation}
where $u$ denotes the membrane potential, $v$ the intermediate membrane state, $w$ the synaptic weights, and $s$ the input spike.
Indices $i$ and $j$ represent the post- and pre-synaptic neurons, respectively, and $l$ refers to the layer index.
The parameters $\tau$ and $t$ indicate the membrane time constant and discrete timestep.

A spike is emitted when the membrane potential surpasses a predefined threshold:
\begin{equation} \label{eq:spike_func}
s_{i}^{l}[t] = H\left(u_{i}^{l}[t] - V_{\textrm{th}}\right) \textrm{,}
\end{equation}
where H($\cdot$) is the Heaviside step function and $V_{\textrm{th}}$ is the firing threshold.

After spike firing, the membrane potential is reset based on the intermediate state using the following mechanism:
\begin{equation} \label{eq:reset_func}
v_{i}^{l}[t] = \left(u_{i}^{l}[t] - s_{i}^{l}[t]\right)s_{i}^{l}[t] + u_{i}^{l}[t]\left(1 - s_{i}^{l}[t]\right) \textrm{.}
\end{equation}

\subsection{Spatio-Temporal Backpropagation (STBP)} \label{app:stbp}
\begin{wrapfigure}{r}{0.3\textwidth}
\vspace{-1.5em}
\begin{center}
    \includegraphics[width=1\linewidth]{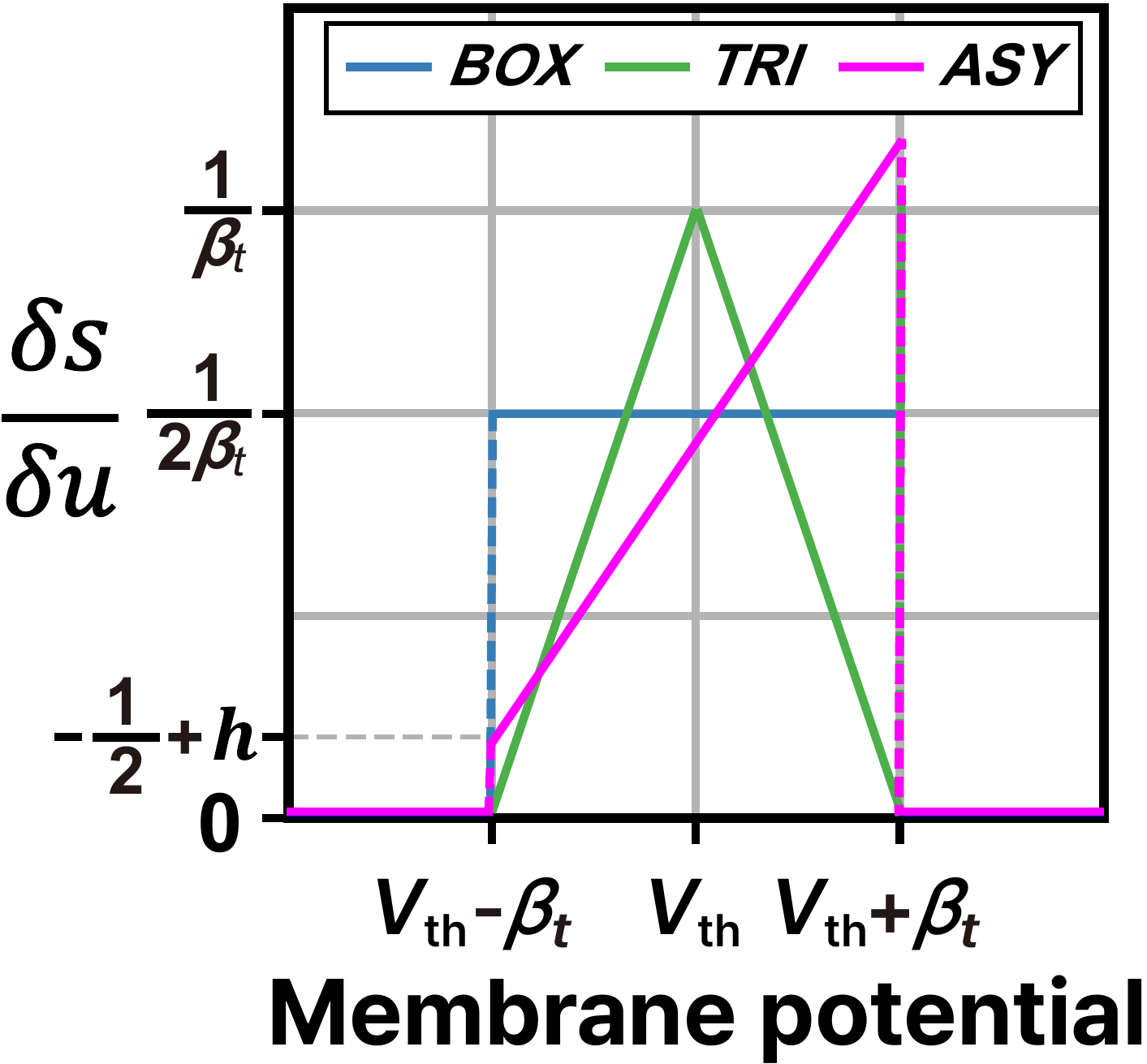}
\end{center}
\vspace{-0.5em}    
\caption{Visualization of surrogate gradient functions (\textit{BOX}, \textit{TRI}, \textit{ASY}).}
\vspace{-1.5em}
\label{fig:SG_function}
\end{wrapfigure}
SNNs require gradient propagation that considers not only spatial but also temporal variations.
STBP is considered a suitable backpropagation method for SNNs as it incorporates both spatial and temporal components.
The gradient of the loss $L$ with respect to the membrane potential $u^i_t[t]$ is defined as follows:
\begin{equation} \label{eq:STBP}
\frac{\partial L}{\partial u^{l}_{i}[t]} = \frac{\partial L}{\partial s^{l}_{i}[t]}\frac{\partial s^{l}_{i}[t]}{\partial u^{l}_{i}[t]} + \frac{\partial L}{\partial u^{l}_{i}[t+1]}\frac{\partial u^{l}_{i}[t+1]}{\partial u^{l}_{i}[t]} \textrm{,}
\end{equation}
here, $u[t]$ is the membrane potential at time $t$, and s[t] is the spike count at time $t$.
The $l$ and $i$ indicate the layer and neuron index, respectively.

In addition, the gradient of the loss for the weight $w^l$ in $l$th layer is given by:

\begin{equation} \label{eq:temporal_grad}
\frac{\partial L}{\partial w^{l}} = \sum_{t} {\frac{\partial L}{\partial w^{l}[t]}} = \sum_{t}{\frac{\partial L}{\partial u^{l}[t]} s^{l-1}[t]} \textrm{.}
\end{equation}

\subsection{Surrogate Gradients in STBP} \label{app:surrogate}

Surrogate gradient replaces the non-differentiable function $\frac {\partial \boldsymbol s}{\partial \boldsymbol u}$ with smooth approximated functions to enable gradient-based training.
Representative surrogate gradient functions include the boxcar function and the triangle function, which are defined as follows:
\begin{equation} \label{ea:gradient_boxcar}
\frac {\partial \boldsymbol s}{\partial \boldsymbol u} = \frac {1}{2\beta}\cdot 1 (|u-V_\textrm{th}|<\beta) \textrm{,}
\end{equation}
\begin{equation} \label{ea:gradient_trinangle}
\frac {\partial \boldsymbol s}{\partial \boldsymbol u} = \frac {1}{\beta^2}\cdot \operatorname{max}(0, \beta-|u-V_\textrm{th}|) \textrm{.}
\end{equation}
The above equations represent the \textit{BOX} function and the \textit{TRI} function, respectively, and $\beta$ is a parameter representing the effective window.
These functions focus on mimicking the Dirac delta function, and their area remains constant regardless of the effective window.
The function shapes are illustrated in Fig.~\ref{fig:SG_function}

\subsection{Fisher Information Matrix (FIM)} \label{app:FIM}

The FIM is defined as:
\begin{equation}
\mathbf{F}(\theta) = \mathbb{E}_{(x,y) \sim p(x,y;\theta)} 
\left[ \nabla_\theta \log p(y \mid x; \theta) \nabla_\theta \log p(y \mid x; \theta)^\top \right] \textrm{,}
\label{eq:FIM}
\end{equation}

where $\theta$ denotes the model parameters, $(x, y)$ is an input-output pair sampled from the data distribution $p(x, y; \theta)$, and $p(y \mid x; \theta)$ is the model’s conditional probability. 
The operator $\nabla_\theta$ represents the gradient with respect to $\theta$.
The FIM measures the variance of the gradient of the log-likelihood, reflecting the sensitivity of the model parameters to changes in the data distribution. 
In this context, the eigenvalues of the FIM quantify the curvature of the loss surface in various parameter directions.

\subsection{Proof of the Derivatives of Symmetric Surrogate Gradients }
\label{app:symmetric_scaling}

Let $f:[\theta-\beta,\theta+\beta]\to\mathbb{R}_{\ge0}$ be a symmetric surrogate function, with a unique maximum at $u=\theta$, and
\[
f(\theta-\beta)=f(\theta+\beta)=m\ (\ge 0), 
\qquad 
\int_{\theta-\beta}^{\theta+\beta} f(u)\,du = c > 0.
\]
Set the \emph{effective area above the floor} by
\[
c_{\mathrm{eff}} \;:=\; \int_{\theta-\beta}^{\theta+\beta} \bigl(f(u)-m\bigr)\,du \;=\; c - 2\beta m .
\]
We assume $c_{\mathrm{eff}}>0$ (otherwise $f\equiv m$ on the window and all derivatives vanish).

Define $g(u):=f(u)-m$. Then $g$ is symmetric, $g(\theta\pm\beta)=0$, and 
$\int_{\theta-\beta}^{\theta+\beta} g(u)\,du=c_{\mathrm{eff}}$. Let $M_f:=f(\theta)$ and $M_g:=g(\theta)=M_f-m$.

\paragraph{Height.}
By symmetry,
\[
c_{\mathrm{eff}} 
= 2\int_{\theta}^{\theta+\beta} g(u)\,du 
\le 2\beta\, M_g
\quad\Rightarrow\quad
M_g \ge \frac{c_{\mathrm{eff}}}{2\beta},
\]
hence
\[
M_f \;=\; m + M_g \;\ge\; m + \frac{c_{\mathrm{eff}}}{2\beta} 
\;=\; \frac{c}{2\beta}.
\]
In particular, $\|f\|_\infty \ge M_f \ge c/(2\beta)$.

\paragraph{First derivative.}
On $[\theta,\theta+\beta]$, $g(\theta)=M_g$ and $g(\theta+\beta)=0$. By the mean value theorem there exists $\xi$ with
\[
|g'(\xi)|=\frac{M_g}{\beta}.
\]
Since $g'=f'$, we obtain
\[
\|f'\|_\infty \;\ge\; \frac{M_g}{\beta} 
\;\ge\; \frac{c_{\mathrm{eff}}}{2\beta^2}
\;=\; \frac{c-2\beta m}{2\beta^2}.
\]

\paragraph{Conclusion.}
With $c_{\mathrm{eff}}=c-2\beta m>0$,
\[
\|f\|_\infty \;\ge\; \frac{c}{2\beta} \;=\; \Omega(\beta^{-1}),
\qquad
\|f'\|_\infty \;\ge\; \frac{c-2\beta m}{2\beta^2} \;=\; \Omega(\beta^{-2}).
\]
Substituting into
\begin{equation} \label{eq:second_deriv_sg}
\frac{\partial^2 L}{\partial w^2}
= \frac{\partial^2 L}{\partial H^2}(H'(u)x)^2
+ \frac{\partial L}{\partial H}\,H''(u)\,x^2
= \frac{\partial^2 L}{\partial H^2}(f(u)x)^2
+ \frac{\partial L}{\partial H}\,f'(u)\,x^2,
\end{equation}
we obtain
\[
\left| \frac{\partial^2 L}{\partial w^2} \right|_{SNN} = \Omega(\frac{x^2}{\beta^2}) \textrm{.}
\]
which shows sharper curvature than the $\mathcal{O}(x^2)$ scaling of smooth DNNs.

\subsection{\textit{TRI} induces larger curvature than \textit{BOX} under area normalization}
\label{app:ls_land_tri_box_comp}

\paragraph{Setup.}
Consider the surrogate gradients supported on $[\,\theta-\beta,\theta+\beta\,]$ with unit area.  
The boxcar (\textit{BOX}) and triangular (\textit{TRI}) surrogates are given by Eqs.~\ref{ea:gradient_boxcar} and~\ref{ea:gradient_trinangle}.
For a weight $w$ with input $x$ and pre-activation $u=wx$, the Hessian contribution is stated in Eq.~\ref{eq:second_deriv_sg}.

\paragraph{\textit{BOX} properties.}
The \textit{BOX} function has support length $2\beta$, constant height $1/(2\beta)$, and area $1$.  
Therefore,
\[
\|f_{\mathrm{BOX}}\|_\infty=\tfrac{1}{2\beta}, 
\qquad \|f'_{\mathrm{BOX}}\|=0 \text{ almost everywhere on }(\theta-\beta,\theta+\beta).
\]

\paragraph{\textit{TRI} properties.}
The \textit{TRI} function is piecewise linear with the maximum at $u=\theta$ given by $f_{\mathrm{TRI}}(\theta)=1/\beta$, twice the peak of \textit{BOX} under the same area constraint.  
The slopes are $\pm 1/\beta^2$, hence
\[
\|f_{\mathrm{TRI}}\|_\infty=\tfrac{1}{\beta},
\qquad \|f'_{\mathrm{TRI}}\|_\infty=\tfrac{1}{\beta^2}.
\]

\paragraph{Comparison.}
Relative to \textit{BOX},
\[
\|f_{\mathrm{TRI}}\|_\infty = 2\,\|f_{\mathrm{BOX}}\|_\infty,
\qquad
\|f'_{\mathrm{TRI}}\|_\infty > \|f'_{\mathrm{BOX}}\|_\infty.
\]
Thus, \textit{TRI} attains both higher peak and steeper slope.

\paragraph{Implication for curvature.}
Substituting into Eq.~\ref{eq:second_deriv_sg}, the first term scales with $f(u)^2$ and is therefore at least four times larger for \textit{TRI} near $u=\theta$ compared to \textit{BOX}.  
The second term involves $f'(u)$, which is strictly larger for \textit{TRI} inside the window since $f'_{\mathrm{BOX}}=0$ almost everywhere.  
As a result, both contributions to the curvature are enhanced under \textit{TRI}, leading to the following conclusion:
\[
\Bigl|\tfrac{\partial^2 L}{\partial w^2}\Bigr|_{\text{TRI}}
\;\ge\; 2 \cdot
\Bigl|\tfrac{\partial^2 L}{\partial w^2}\Bigr|_{\text{BOX}}
\]

\subsection{Adaptive Surrogate Gradients Considering SGV and TGC} \label{app:adaptive}
\begin{algorithm}[h]
\caption{Adaptive Surrogate Gradients Considering SGV and TGC}
\label{alg:adaptive}
\begin{algorithmic}[1]
    \STATE \textbf{Input:} iteration i, $\frac{\partial L}{\partial s[t]}$
    \STATE \textbf{Output:} $\frac{\partial L}{\partial u[t]}$
    
    \FOR{$l = L$ to $1$} 
        \IF{$i=0$} 
            \STATE $\beta^{l}_{[t:0,\dots,T]} \leftarrow \beta_{\textrm{init}}$
        \ENDIF
        \IF{(($i \bmod i_{\textrm{update}}$) = 0) \textrm{and} ($i\neq0$)}
            \FOR{$t = T$ to $1$}
                \IF{$t=T$}
                    \STATE $\beta^{l}_{t} \leftarrow \beta_{search} (SGV, \frac{\partial L}{\partial s[t]}, u^{l}[t])$
                \ELSE
                    \STATE $\beta^{l}_{t} \leftarrow \beta_{search} (TGC, g^{l}_{\textrm{cur}}[t], \frac{\partial L}{\partial s[t]}, u^{l}[t])$
                \ENDIF
            \ENDFOR    
        \ELSE
            \FOR{$t==T$ to $1$}
                \STATE $g^{l}_{\textrm{cur}}[t] \leftarrow f(u^{l}[t],\beta^{l}_{t})$ 
            \ENDFOR
        \ENDIF
        
    \ENDFOR
\end{algorithmic}
\end{algorithm}
Alg.~\ref{alg:adaptive} outlines the adaptive surrogate gradient procedure that simultaneously considers both SGV and TGC.
At the beginning of training, the effective window $\beta$ is initialized. 
During every update step, $\beta$ is adaptively adjusted according to the observed gradient distribution: SGV is used to calibrate the final timestep, while TGC is employed for earlier timesteps to maintain temporal consistency of error signals.
When not in an update step, the algorithm computes the current layer gradient using the stored $\beta$.
In this way, the method dynamically tunes $\beta$ across layers and timesteps.

\subsection{Beta Search via Bayesian Optimization} \label{app:beta_search}
\begin{algorithm}[h!]
\caption{Beta Search via Bayesian Optimization}
\label{alg:beta_search}
\begin{algorithmic}[1]
    \STATE \textbf{Input:} metric $M$, $g_{\mathrm{ref}}$, $\frac{\partial L}{\partial s[t]}$, $u[t]$, $t$, $j$
    \STATE \textbf{Output:} $\beta$

    \STATE $\beta_{\mathrm{min}}  \ \leftarrow \beta_{\mathrm{min,init}}$, $\beta_{\mathrm{max}}\leftarrow \beta_{\mathrm{max,init}}$
    
        \STATE $\boldsymbol{\beta}_{\mathrm{obs}} \gets \mathrm{Random}(\beta_{\mathrm{min}}, \beta_{\mathrm{max}}, n_{\mathrm{obs}})$
        \IF{$t = T$}
            \STATE $\boldsymbol{M} \gets SGV(\boldsymbol{\beta}_{\mathrm{obs}}, u[t], 
            \frac{\partial L}{\partial s[t]})$ 
        \ELSE
            \STATE $\boldsymbol{M} \gets TGC(\boldsymbol{\beta}_{\mathrm{obs}}, g_{ref}, u[t], \frac{\partial L}{\partial s[t]})$ 
        \ENDIF
        \STATE $f_{\mathrm{best}} \gets \max(\boldsymbol{M})$
        \STATE $\beta_{best} \gets \beta_{obs}[\operatorname{argmax(M)}]$

        \STATE $\beta_{min} \gets \operatorname{max}(\beta_{best} - \delta,\beta_\mathrm{min}$)
        \STATE $\beta_{max} \gets \operatorname{min}(\beta_{best} + \delta,\beta_\mathrm{max}$)

        \STATE $\boldsymbol{\beta}_{\mathrm{eval}} \gets \mathrm{Uniform}(\beta_{\mathrm{min}}, \beta_{\mathrm{max}}, n_{\mathrm{eval}})$

        \STATE $\mu_s, \sigma_s \gets \mathrm{GP}(\boldsymbol{\beta}_{\mathrm{obs}}, \boldsymbol{M}, \boldsymbol{\beta}_{\mathrm{eval}})$

        \STATE $\mathrm{EI} \gets \mathrm{ExpectedImprovement}(\mu_s, \sigma_s, f_{\mathrm{best}})$

        \STATE $\beta^* \gets \arg\max_{\beta} \mathrm{EI}$

        \IF{$\beta^* < f_\mathrm{best}$}
            \STATE$\beta^* \gets \beta_{\mathrm{best}}$
        \ENDIF
    
    \STATE $\beta \gets \beta^*$
\end{algorithmic}
\end{algorithm}

Alg.~\ref{alg:beta_search} describes the adaptive search strategy for the effective window $\beta$. 
The method initializes a search range and samples candidates to evaluate the training metric.
Depending on the current timestep \textit{t}, either SGV or TGC is used to compute the evaluation metric. 
A Gaussian Process (GP) surrogate model is then fitted over the observed results, and the expected improvement (EI) criterion guides the selection of the next candidate $\beta$.
The algorithm iteratively narrows down the search interval around the best candidate, ensuring efficient exploration while avoiding unstable regions. 
The final $\beta$ is set to the candidate with the highest improvement score.
In our implementation, we fit a Gaussian Process surrogate with an RBF kernel, and the next candidate is proposed using the Expected Improvement criterion.

\subsection{Theoretical Analysis}

\begin{theorem}[CV-Minimizing Symmetric Function under Area and Boundary Constraints]~\label{th:app_cv_min_syn_func}
Let $f: [a, b] \to \mathbb{R}_{\ge 0}$ be a function satisfying, where $a=\theta-\beta$ and $b = \theta + \beta$:
\begin{itemize}
    \item \textbf{(Symmetry)}: \( f(u) = f(a + b - u) \) for all \( u \in [a,b] \)
    \item \textbf{(Boundary condition)}: \( f(a) = f(b) = 0 \)
    \item \textbf{(Nonnegativity)}: \( f(u) \ge 0 \)
    \item \textbf{(Area constraint)}: \( \int_a^b f(u)\,du = c > 0 \)
\end{itemize}
Let the weight function be given by the Gaussian density
\[
p(u) := \frac{1}{\sqrt{2\pi}\sigma} \exp\left( -\frac{(u - \mu)^2}{2\sigma^2} \right), \quad \text{with } \mu < a,
\]
so that \( p(u) \) is strictly decreasing and convex on \( [a, b] \).

Then the unique function \( f^* \) that minimizes the CV
\[
\mathbf{CV_{sim}}[f] := \frac{ \sqrt{ \int_a^b f(u)^2 p(u)\,du - \left( \int_a^b f(u) p(u)\,du \right)^2 } }{ \int_a^b f(u) p(u)\,du }, 
\]
over all such admissible functions is the symmetric triangular function
\[
f^*(u) := \frac{4c}{(b - a)^2} \cdot \min(u - a,\ b - u).
\]
\end{theorem}
\begin{proof} 
Let $m = \frac{a + b}{2}$ be the midpoint of the interval.
Due to the symmetry condition, any admissible function $f$ is uniquely determined by its restriction $g$ to $[a, m]$, with the reconstruction:
\[
f(u) =
\begin{cases}
g(u), & u \in [a, m], \\
g(a + b - u), & u \in [m, b].
\end{cases}
\]
Since \( p(u) \) is strictly decreasing and convex on \( [a, b] \), it is also decreasing on \( [a, m] \). Over this domain, define the following Rayleigh-type quotient:
\[
R[g] := \frac{ \int_a^m g(u)^2 p(u)\,du + \int_m^b g(a + b - u)^2 p(u)\,du }{ \left( \int_a^m g(u) p(u)\,du + \int_m^b g(a + b - u) p(u)\,du \right)^2 }.
\]
By change of variable \( u' = a + b - u \), and using symmetry of \( f \), we have:
\[
R[g] = \frac{ 2 \int_a^m g(u)^2 p(u)\,du }{ \left( 2 \int_a^m g(u) p(u)\,du \right)^2 } = \frac{ \int_a^m g(u)^2 p(u)\,du }{ \left( \int_a^m g(u) p(u)\,du \right)^2 }.
\]

This Rayleigh quotient is minimized when \( g(u) \) is proportional to a linear function increasing from \( a \) to \( m \). That is:
\[
g^*(u) = \alpha(u - a), \quad u \in [a, m],
\]
with boundary condition \( g(a) = 0 \). Extending symmetrically gives:
\[
f^*(u) = \alpha \cdot \min(u - a, b - u).
\]
To satisfy the area constraint:
\[
\int_a^b f^*(u)\,du = \int_a^m \alpha(u - a)\,du + \int_m^b \alpha(b - u)\,du = \alpha \cdot \frac{(b - a)^2}{4} = c,
\]
which yields:
\[
\alpha = \frac{4c}{(b - a)^2}.
\]
Thus, the minimizing function is:
\[
f^*(u) = \frac{4c}{(b - a)^2} \cdot \min(u - a, b - u),
\]
which is the symmetric triangular function.
\end{proof}

\begin{theorem}[CV Comparison of Asymmetric and Symmetric Surrogates]~\label{th:app_cv_comp}
Let $f_{\mathrm{asy}}(u)$ and $f_{\mathrm{sym}}(u)$ be asymmetric and symmetric surrogate gradient functions defined over $[a, b]$, satisfying the boundary condition $f(a)=f(b)=0$, nonnegativity $f(u)\geq0$, and area constraint $\int_a^b f(u)\,du = c$, where $a=\theta-\beta$ and $b = \theta + \beta$.
Suppose the membrane potential $u \sim \mathcal{N}(\mu, \sigma^2)$ with $\mu < a$, so that $p(u)$ is strictly decreasing on $[a, b]$.
Then, under a linear approximation of the Gaussian, we have:
\[
\mathrm{CV}_{\mathrm{asy}} < \mathrm{CV}_{\mathrm{sym}} \quad \text{if } L \kappa> \sigma^2.
\]
\end{theorem}

\begin{proof}
We define \( L := b - a \) and approximate the Gaussian by a first-order Taylor expansion around $u=a$:
\[
A := p(a) = \frac{1}{\sqrt{2\pi} \sigma} \exp\left(-\frac{(a - \mu)^2}{2\sigma^2}\right), \quad
\]

\[
B := -p'(a) = \frac{a - \mu}{\sigma^2} \cdot A.
\]
Let \( \kappa:= a - \mu \), so that:
\[
A = \frac{1}{\sqrt{2\pi} \sigma} e^{-\frac{\kappa^2}{2\sigma^2}}, \quad B = \frac{\kappa}{\sigma^2} A.
\]

Under the area constraint, define the surrogates:
\begin{align*}
f_{\mathrm{asy}}(u) &= \frac{2c}{L} \cdot \frac{u - a}{L}, \\
\text{and} \quad
f_{\mathrm{sym\_tri}}(u) &=
\begin{cases}
\frac{2c}{L} \cdot \frac{u - a}{L/2}, & u \le \theta, \\
\frac{2c}{L} \cdot \frac{b - u}{L/2}, & u > \theta,
\end{cases}
\quad \text{where } \theta = \frac{a + b}{2}.
\end{align*}

\textbf{Asymmetric surrogate:}
\[
f_{\mathrm{asy}}(u) = \frac{2c}{L^2}(u - a)
\]

\textit{Expectation:}
\begin{align*}
\mathbb{E}[f_{\mathrm{asy}}]
&= \int_a^b f_{\mathrm{asy}}(u) \cdot p(u)\,du \\
&= \frac{2c}{L^2} \int_a^b (u - a)(A - B(u - a))\,du \\
&= \frac{2c}{L^2} \int_0^L x(A - Bx)\,dx \\
&= \frac{2c}{L^2} \left( A \cdot \frac{L^2}{2} - B \cdot \frac{L^3}{3} \right) \\
&= c \left( A - \frac{2}{3} B L \right)
\end{align*}

\textit{Second moment:}
\begin{align*}
\mathbb{E}[f_{\mathrm{asy}}^2]
&= \left( \frac{2c}{L^2} \right)^2 \int_0^L x^2 (A - Bx)\,dx \\
&= \frac{4c^2}{L^4} \left( A \cdot \frac{L^3}{3} - B \cdot \frac{L^4}{4} \right) \\
&= \frac{4c^2}{L} \left( \frac{A}{3} - \frac{B L}{4} \right)
\end{align*}

\textit{Variance:}
\begin{align*}
\mathrm{Var}[f_{\mathrm{asy}}]
&= \mathbb{E}[f_{\mathrm{asy}}^2] - \mathbb{E}[f_{\mathrm{asy}}]^2 \\
&= \frac{4c^2}{L} \left( \frac{A}{3} - \frac{B L}{4} \right)
- c^2 \left( A - \frac{2}{3} B L \right)^2
\end{align*}

\vspace{0.5em}
\textbf{Symmetric surrogate (triangle function):}
\[
f_{\mathrm{sym\_tri}}(u) =
\begin{cases}
\frac{4c}{L^2}(u - a), & u \in [a, \theta], \\
\frac{4c}{L^2}(b - u), & u \in [\theta, b]
\end{cases}, \quad \theta = \frac{a + b}{2}
\]

\textit{Expectation:}
\begin{align*}
\mathbb{E}[f_{\mathrm{sym\_tri}}]
&= \frac{4c}{L^2} \left[ \int_a^\theta (u - a)p(u)\,du + \int_\theta^b (b - u)p(u)\,du \right] \\
&= \frac{8c}{L^2} \int_0^{L/2} x(A - Bx)\,dx \\
&= \frac{8c}{L^2} \left( A \cdot \frac{(L/2)^2}{2} - B \cdot \frac{(L/2)^3}{3} \right) \\
&= c \left( A - \frac{1}{3} B L \right)
\end{align*}

\textit{Second moment:}
\begin{align*}
\mathbb{E}[f_{\mathrm{sym\_tri}}^2]
&= \frac{8c^2}{L^4} \int_0^{L/2} x^2(A - Bx)\,dx \cdot 2 \\
&= \frac{16c^2}{L^4} \left( A \cdot \frac{(L/2)^3}{3} - B \cdot \frac{(L/2)^4}{4} \right) \\
&= \frac{2c^2}{L} \left( \frac{A}{3} - \frac{B L}{8} \right)
\end{align*}

\textit{Variance:}
\[
\mathrm{Var}[f_{\mathrm{sym\_tri}}] = \frac{2c^2}{L} \left( \frac{A}{3} - \frac{B L}{8} \right)
- c^2 \left( A - \frac{1}{3} B L \right)^2
\]

\vspace{0.5em}
\textbf{Coefficient Ratio:}
\[
r := \frac{\mathbb{E}[f_{\mathrm{asy}}]}{\mathbb{E}[f_{\mathrm{sym\_tri}}]} = \frac{A - \frac{2}{3} B L}{A - \frac{1}{3} B L}
= \frac{3A - 2B L}{3A - B L}
\]

\[
\eta := \frac{\mathrm{Var}[f_{\mathrm{asy}}]}{\mathrm{Var}[f_{\mathrm{sym\_tri}}]} \\
= \frac{
\frac{4c^2}{L} \left( \frac{A}{3} - \frac{B L}{4} \right) - c^2 \left( A - \frac{2}{3} B L \right)^2
}{
\frac{2c^2}{L} \left( \frac{A}{3} - \frac{B L}{8} \right) - c^2 \left( A - \frac{1}{3} B L \right)^2
}
\]

Then the squared CV ratio becomes:
\begin{align*}
\left( \frac{\mathrm{CV}_{\mathrm{asy}}}{\mathrm{CV}_{\mathrm{sym\_tri}}} \right)^2 &= \frac{\eta}{r^2} \\
& =
\frac{
3(2A - B L)^2 \cdot \left[12A - 9B L + L(3A - 2B L)^2\right]
}{
(3A - 2B L)^2 \cdot \left[16A - 8B L + 3L(2A - B L)^2\right]
} ,
\end{align*}
which simplifies under substitution to:
\[
\frac{\eta}{r^2} \sim \frac{3}{4} \cdot \left( \frac{L\kappa- 2\sigma^2}{2L\kappa- 3\sigma^2} \right)^2
\]
and thus yields the condition \( \mathrm{CV}_{\mathrm{asy}} < \mathrm{CV}_{\mathrm{sym\_tri}} \) if \( L \kappa> \sigma^2 \).
Therefore, according to the Theorem~\ref{th:cv_min_syn_func}, \( \mathrm{CV}_{\mathrm{asy}} < \mathrm{CV}_{\mathrm{sym}} \) if \( L \kappa> \sigma^2 \).

\end{proof}

\subsection{Extension of Theorem~\ref{th:cv_comp} with $n$-segment piecewise-linear approximation}
\label{app:pwla_stability}

In this subsection, we show that the conclusion of Theorem~\ref{th:cv_comp} is stable when the
Gaussian weight function on the effective window is approximated by an
$n$-segment piecewise-linear model instead of a single linear segment.
We use the same notation as in Theorem~\ref{th:cv_comp} and Theorem~\ref{th:app_cv_comp}: the effective window is
$I = [a,b] = [\theta-\beta,\theta+\beta]$ with width $L = b-a$, the membrane
potential satisfies $u \sim \mathcal N(\mu,\sigma^2)$ with $\mu < a$, and
$p(u)$ denotes the Gaussian weight function.
We also define $\kappa := a-\mu > 0$.

\paragraph{One-segment and $n$-segment PWLA models.}
We now specify the weight-function models on $I$.
\begin{itemize}
    \item \textbf{One-segment linear model $\tilde p_1$.}
    Following Theorem~\ref{th:cv_comp} (and Theorem~\ref{th:app_cv_comp}), we approximate $p(u)$ on $I$ by its
    first-order Taylor expansion at the left boundary $u=a$:
    \begin{equation*}
      \tilde p_1(u)
      := p(a) + p'(a)\,(u-a),
      \qquad u\in I.
    \end{equation*}

    \item \textbf{$n$-segment piecewise-linear model $\tilde p_n$.}
    For $n\in\mathbb N$, we divide $I$ into $n$ equal sub-intervals
    of length $\Delta := L/n$ with grid points $u_k := a + k\Delta$
    for $k = 0,\dots,n$. The $n$-segment piecewise-linear approximation
    of $p$ is defined by
    \begin{equation}
      \tilde p_n(u)
      :=
      p(u_k) + p'(u_k)\,(u-u_k)
      \quad \text{for } u \in [u_k,u_{k+1}),\; k = 0,\dots,n-1.
    \end{equation}
\end{itemize}
We denote $R_1^2 := R(\tilde p_1)^2$ and $R_n^2 := R(\tilde p_n)^2$.

We denote by $\mathcal{H}$ the class of weight functions considered in this paper,
namely the Gaussian weight $p(u)$ and its PWLA approximations $\tilde p_1,\tilde p_n$
on the effective window $I=[\theta-\beta,\theta+\beta]$
We refer to elements of $\mathcal{H}$ as admissible weight functions.

\begin{lemma}[Bound on the variation of $R(h)^2$]
\label{lem:R_lipschitz}
Let $I = [a,b]$ be the effective window.
For any nonnegative weight function $h:I\to\mathbb{R}_{\ge 0}$, define
\begin{align*}
M_1(h) &= \int_a^b f_{\mathrm{asy}}(u)\,h(u)\,du,
&
M_2(h) &= \int_a^b f_{\mathrm{asy}}(u)^2\,h(u)\,du, \\
J_1(h) &= \int_a^b f_{\mathrm{sym}}(u)\,h(u)\,du,
&
J_2(h) &= \int_a^b f_{\mathrm{sym}}(u)^2\,h(u)\,du.
\end{align*}
Based on these, we define
\begin{align*}
\mathrm{CV}_{\mathrm{asy}}(h)^2
&:= \frac{M_2(h)-M_1(h)^2}{M_1(h)^2}, \\
\mathrm{CV}_{\mathrm{sym}}(h)^2
&:= \frac{J_2(h)-J_1(h)^2}{J_1(h)^2}, \\
R(h)^2
&:= \left(\frac{\mathrm{CV}_{\mathrm{asy}}(h)}{\mathrm{CV}_{\mathrm{sym}}(h)}\right)^2.
\end{align*}
Assume that there exist constants $m_{\min}>0$ and $v_{\min}>0$ such that, for all admissible
weight functions $h$ in the class considered in this paper,
\begin{equation}
\label{eq:lemma_assumptions}
M_1(h),\,J_1(h) \ge m_{\min},
\qquad
M_2(h)-M_1(h)^2,\; J_2(h)-J_1(h)^2 \ge v_{\min}.
\end{equation}
Then there exists a constant $K>0$ such that, for any admissible $h,\tilde h$,
\begin{equation}
\label{eq:R_lip_bound}
|R(h)^2 - R(\tilde h)^2|
\;\le\; K\,\|h-\tilde h\|_\infty,
\end{equation}
where $\|h-\tilde h\|_\infty := \sup_{u\in I}|h(u)-\tilde h(u)|$.
\end{lemma}

\begin{proof}
We denote $R(h)^2$ as
\[
R(h)^2 = \Phi\bigl(v(h)\bigr),
\]
where
\[
v(h) := \bigl(M_1(h),M_2(h),J_1(h),J_2(h)\bigr)\in\mathbb{R}^4 \textrm{,}
\]
and
\[
\Phi(m_1,m_2,j_1,j_2)
:= \frac{j_1^2}{m_1^2}
   \cdot
   \frac{m_2-m_1^2}{j_2-j_1^2}.
\]

\paragraph{Step 1: Sensitivity of $v(h)$ with respect to $h$.}
Let $L := b-a$ and let $h,\tilde h$ be two admissible weight functions.
Set $\Delta h(u) := h(u)-\tilde h(u)$.
Since $f_{\mathrm{asy}}$ and $f_{\mathrm{sym}}$ are fixed surrogate functions on the effective window $I$, they are bounded:
\[
C_1 := \sup_{u\in I}\bigl(|f_{\mathrm{asy}}(u)|,\,|f_{\mathrm{sym}}(u)|\bigr) < \infty,
\quad
C_2 := \sup_{u\in I}\bigl(f_{\mathrm{asy}}(u)^2,\,f_{\mathrm{sym}}(u)^2\bigr) < \infty.
\]

We first bound the change in $M_1$:
\begin{align*}
M_1(h) - M_1(\tilde h)
&= \int_a^b f_{\mathrm{asy}}(u)\,h(u)\,du
  - \int_a^b f_{\mathrm{asy}}(u)\,\tilde h(u)\,du \\
&= \int_a^b f_{\mathrm{asy}}(u)\,\Delta h(u)\,du.
\end{align*}
Using the triangle inequality and the definition of the sup norm,
\begin{align*}
|M_1(h) - M_1(\tilde h)|
&\le \int_a^b |f_{\mathrm{asy}}(u)|\,|\Delta h(u)|\,du \\
&\le \left(\sup_{u\in I}|f_{\mathrm{asy}}(u)|\right)
      \int_a^b |\Delta h(u)|\,du \\
&\le C_1 \cdot L \cdot \sup_{u\in I}|\Delta h(u)| \\
&= C_1 L\,\|h-\tilde h\|_\infty.
\end{align*}
Similarly, for $M_2$ we have
\begin{align*}
M_2(h) - M_2(\tilde h)
&= \int_a^b f_{\mathrm{asy}}(u)^2\,\Delta h(u)\,du,
\end{align*}
and hence
\[
|M_2(h) - M_2(\tilde h)|
\le C_2 L\,\|h-\tilde h\|_\infty.
\]
The same argument with $f_{\mathrm{sym}}$ and $f_{\mathrm{sym}}^2$ yields
\[
|J_1(h) - J_1(\tilde h)| \le C_1 L\,\|h-\tilde h\|_\infty,
\qquad
|J_2(h) - J_2(\tilde h)| \le C_2 L\,\|h-\tilde h\|_\infty.
\]

Therefore, there exists a constant $C_{v}>0$ ($C_{v} := L\max\{C_1,C_2\}$) such that
\begin{equation}
\label{eq:v_diff_bound}
\|v(h) - v(\tilde h)\|
\;\le\; C_{v}\,\|h-\tilde h\|_\infty,
\end{equation}
where $\|\cdot\|$ denotes the Euclidean norm on $\mathbb{R}^4$.

\paragraph{Step 2: Sensitivity of $\Phi(v)$ with respect to $v$.}

By the assumptions in \eqref{eq:lemma_assumptions}, for admissible
weight functions $h$, we have
\[
M_1(h),J_1(h) \ge m_{\min},
\qquad
M_2(h)-M_1(h)^2,\; J_2(h)-J_1(h)^2 \ge v_{\min}.
\]
Moreover, since the surrogates and weight functions are bounded
on the finite interval $I$, the integrals $M_1(h),M_2(h),J_1(h),J_2(h)$
are bounded.
Hence all vectors $v(h)$ lie in a set
$K \subset \mathbb{R}^4$ on which:
\begin{itemize}
  \item the denominators $m_1^2$ and $j_2 - j_1^2$ in the definition of
        $\Phi(m_1,m_2,j_1,j_2)$ are bounded away from zero by
        $m_{\min}^2$ and $v_{\min}$, and
  \item $\Phi$ is continuously differentiable.
\end{itemize}
In particular, the gradient is bounded on $K$:
\[
  C_\Phi := \sup_{v\in K} \|\nabla\Phi(v)\| < \infty.
\]

Now, for any $v,\tilde v\in K$, consider the line segment
$\gamma(t) := v + t(\tilde v - v)$, $t\in[0,1]$, and define
$\psi(t) := \Phi(\gamma(t))$. By the chain rule,
\[
\psi'(t) = \nabla\Phi(\gamma(t))\cdot(\tilde v - v),
\]
so that
\[
|\psi'(t)| \le \|\nabla\Phi(\gamma(t))\|\,\|\tilde v - v\|
\le C_\Phi\,\|\tilde v - v\|
\quad\text{for all }t\in[0,1].
\]

Integrating from $0$ to $1$, we obtain
\begin{align*}
|\Phi(\tilde v) - \Phi(v)|
&= |\psi(1) - \psi(0)| \\
&= \left|\int_0^1 \psi'(t)\,dt\right| \\
&\le \int_0^1 |\psi'(t)|\,dt \\
&\le \int_0^1 C_\Phi\,\|\tilde v - v\|\,dt \\
&= C_\Phi\,\|\tilde v - v\|.
\end{align*}
Therefore
\begin{equation}
\label{eq:Phi_diff_bound}
|\Phi(\tilde v) - \Phi(v)| \;\le\; C_\Phi\,\|\tilde v - v\|.
\end{equation}

\paragraph{Step 3: Combining the two bounds.}
Finally, for $h,\tilde h$ we have
\[
|R(h)^2 - R(\tilde h)^2|
= |\Phi(v(h)) - \Phi(v(\tilde h))|.
\]
Applying \eqref{eq:Phi_diff_bound} with $v=v(h)$ and $\tilde v=v(\tilde h)$, and then
\eqref{eq:v_diff_bound}, we obtain
\[
|R(h)^2 - R(\tilde h)^2|
\le C_\Phi \,\|v(h)-v(\tilde h)\|
\le C_\Phi C_{v}\,\|h-\tilde h\|_\infty \textrm{,}
\]
where $K = C_\Phi C_{v}$.
\end{proof}

\begin{lemma}[Approximation error of $\tilde p_1$ and $\tilde p_n$]
\label{lem:pwla_error}
Let $I=[a,b]$ be the effective window with length $L=b-a$.
Assume that $p$ is twice continuously differentiable on $I$, and define
\[
C_{p''} := \sup_{u\in I} |p''(u)| < \infty.
\]
Then, for all $u\in I$,
\[
|p(u) - \tilde p_1(u)| \;\le\; \tfrac{1}{2} C_{p''} L^2,
\qquad
|p(u) - \tilde p_n(u)| \;\le\; \tfrac{1}{2} C_{p''} \frac{L^2}{n^2},
\]
and hence
\begin{equation} \label{eq:pwla_error}
\|\tilde p_1 - \tilde p_n\|_\infty
\;\le\;
\tfrac{1}{2} C_{p''} L^2\Bigl(1 + \tfrac{1}{n^2}\Bigr).
\end{equation}
\end{lemma}

\paragraph{Error of the one-segment linear approximation $\tilde p_1$.}
Let $I=[a,b]$ and $L=b-a$.
The one-segment linear model $\tilde p_1$ is defined as the first-order
Taylor polynomial of $p$ at $u=a$:
\[
\tilde p_1(u) := p(a) + p'(a)(u-a), \qquad u\in I.
\]
By Taylor's theorem with Lagrange remainder at $u=a$, for each $u\in I$
there exists a point $\xi_u$ on the segment between $a$ and $u$ such that
\[
p(u)
= p(a) + p'(a)(u-a) + \frac{1}{2}p''(\xi_u)(u-a)^2.
\]
Hence
\[
p(u) - \tilde p_1(u)
= \frac{1}{2}p''(\xi_u)(u-a)^2.
\]
Taking absolute values and using the definition of $C_{p''}$,
\[
|p(u) - \tilde p_1(u)|
= \frac{1}{2} |p''(\xi_u)|\,|u-a|^2
\le \frac{1}{2} C_{p''}\,|u-a|^2.
\]
Since $u\in[a,b]$ implies $|u-a|\le L$, we obtain
\[
|p(u) - \tilde p_1(u)|
\le \frac{1}{2} C_{p''} L^2
\qquad\text{for all }u\in I.
\]

\paragraph{Error of the $n$-segment PWLA approximation $\tilde p_n$.}
Partition $I=[a,b]$ into $n$ equal sub-intervals of length $\Delta := L/n$,
and let $u_k := a + k\Delta$ for $k=0,\dots,n$.
On each sub-interval $[u_k,u_{k+1}]$, the $n$-segment PWLA model is
\[
\tilde p_n(u) := p(u_k) + p'(u_k)(u-u_k),
\qquad u\in [u_k,u_{k+1}].
\]
Applying Taylor's theorem at $u_k$ for $u\in[u_k,u_{k+1}]$, there exists
$\xi_{k,u}\in[u_k,u_{k+1}]$ such that
\[
p(u)
= p(u_k) + p'(u_k)(u-u_k) + \frac{1}{2}p''(\xi_{k,u})(u-u_k)^2.
\]
Therefore
\[
p(u) - \tilde p_n(u)
= \frac{1}{2}p''(\xi_{k,u})(u-u_k)^2,
\]
and thus
\[
|p(u) - \tilde p_n(u)|
= \frac{1}{2}|p''(\xi_{k,u})|\,|u-u_k|^2
\le \frac{1}{2}C_{p''}|u-u_k|^2.
\]
Since $u\in[u_k,u_{k+1}]$ implies $|u-u_k|\le\Delta=L/n$, we obtain
\[
|p(u) - \tilde p_n(u)|
\le \frac{1}{2}C_{p''}\Delta^2
= \frac{1}{2}C_{p''}\frac{L^2}{n^2}
\qquad\text{for all }u\in I.
\]

\paragraph{Sup-norm bound between $\tilde p_1$ and $\tilde p_n$.}
Taking the supremum over $u\in I$ in the pointwise bounds above gives
\[
\|\tilde p_1 - p\|_\infty \le \tfrac{1}{2}C_{p''}L^2,
\qquad
\|p - \tilde p_n\|_\infty \le \tfrac{1}{2}C_{p''}\frac{L^2}{n^2}.
\]
By the triangle inequality in the sup norm,
\[
\|\tilde p_1 - \tilde p_n\|_\infty
\;\le\;
\|\tilde p_1 - p\|_\infty + \|p - \tilde p_n\|_\infty,
\]
then, 
\[
\|\tilde p_1 - \tilde p_n\|_\infty
\;\le\;
\|\tilde p_1 - p\|_\infty + \|p - \tilde p_n\|_\infty,
\;\le\;
\tfrac{1}{2} C_{p''} L^2\Bigl(1 + \tfrac{1}{n^2}\Bigr).
\]

%
\begin{corollary}[Robustness of Theorem~\ref{th:cv_comp} under $n$-segment PWLA]
\label{cor:pwla_stability}
Suppose that, on the parameter range considered in this paper, the one-segment
linear model $\tilde p_1$ used in Theorem~\ref{th:cv_comp} satisfies
\[
R(\tilde p_1)^2 \;\le\; 1 - \delta_0
\]
for some margin $\delta_0>0$ (for example, under the condition
$L\kappa > \sigma^2$ in Theorem~\ref{th:cv_comp}). Let $\tilde p_n$ be the $n$-segment
piecewise-linear model defined above, and assume the conditions of
Lemmas~\ref{lem:R_lipschitz} and \ref{lem:pwla_error} hold.
Then, for all $n\in\mathbb N$,
\[
\bigl|R(\tilde p_n)^2 - R(\tilde p_1)^2\bigr|
\;\le\;
K\,\|\tilde p_n - \tilde p_1\|_\infty
\;\le\;
\frac{1}{2} K C_{p''} L^2\Bigl(1 + \tfrac{1}{n^2}\Bigr),
\]
where $K>0$ is the Lipschitz constant from Lemma~\ref{lem:R_lipschitz}
(measuring the sensitivity of $R(h)^2$ to perturbations of the weight $h$)
and $C_{p''} := \sup_{u\in I}|p''(u)|$ is the curvature bound from
Lemma~\ref{lem:pwla_error} (quantifying how strongly the Gaussian weight
can bend on the effective window).
In particular, if the effective window satisfies
\[
K C_{p''} L^2 \;\le\; \delta_0,
\]
then
\[
R(\tilde p_n)^2
\;\le\; R(\tilde p_1)^2 + \frac{\delta_0}{2}
\;\le\; 1 - \frac{\delta_0}{2}
\;<\; 1.
\]
Hence, the inequality $\mathrm{CV}_{\mathrm{asy}} < \mathrm{CV}_{\mathrm{sym}}$
derived under the one-segment linear model in Theorem~\ref{th:cv_comp} can be valid for any
sufficiently accurate $n$-segment piecewise-linear approximation $\tilde p_n$
of the same Gaussian weight function on the effective window.
\end{corollary}

\subsection{Empirical Validation} \label{app:empirical}
\begin{figure}[h]
\vspace{-0.5em}
\begin{center}
    \includegraphics[width=0.5\linewidth]{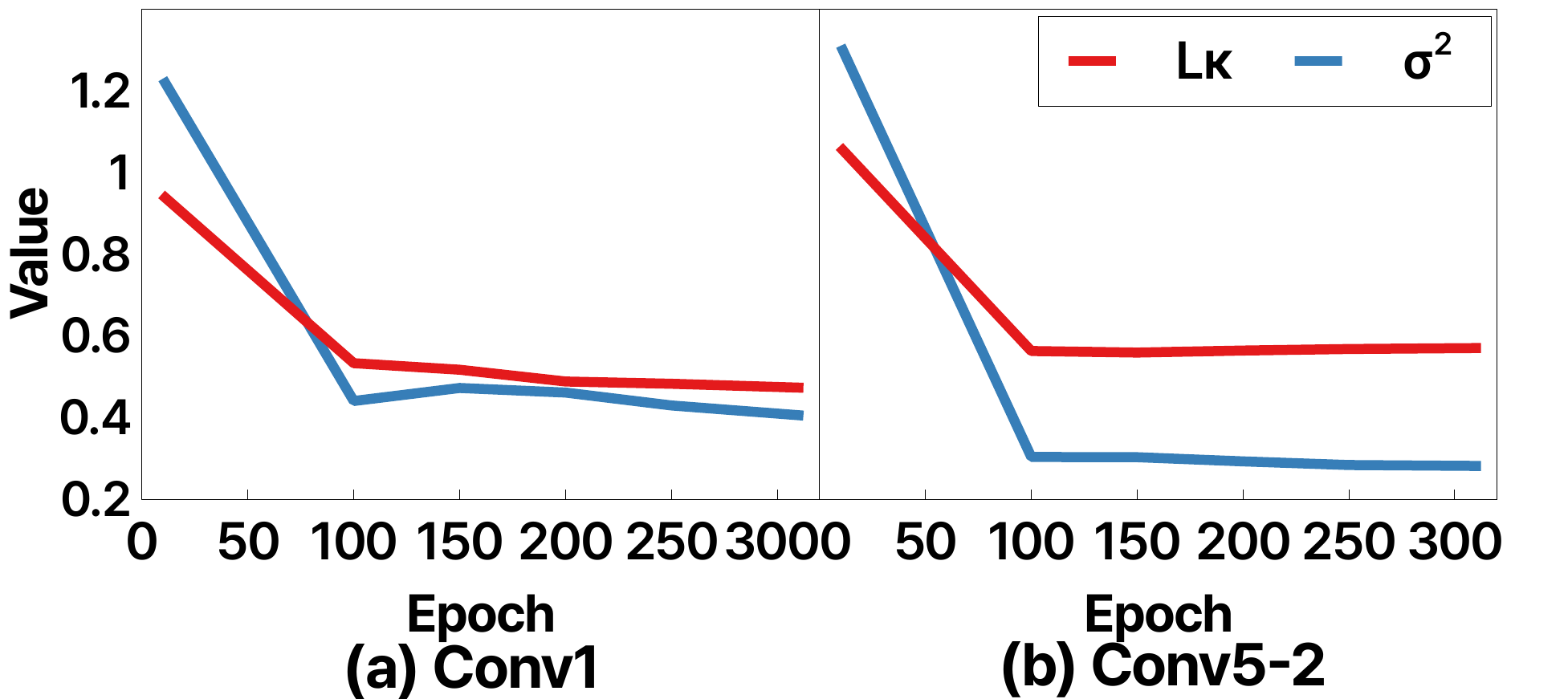}
\end{center}
\caption{\textit{L$\kappa$} and $\sigma^2$ values across epochs. (a) and (b) correspond to Conv1 and Conv5-2, respectively.}
\label{fig:theoretical_validation}
\end{figure}

We derived a sufficient condition for the proposed surrogate to achieve a lower gradient variation than symmetric surrogates, namely $L\kappa > \sigma^2$, where $L=b-a$ and $\kappa=a-\mu$ under the effective window $[a,b]=[\theta-\beta,\theta+\beta]$.
To examine whether this condition is supported in practice, we track $L\kappa$ and $\sigma^2$ throughout training.
Fig.~\ref{fig:theoretical_validation} reports the values of $L\kappa$ and $\sigma^2$ for Conv1 and Conv5-2 of VGG16 on CIFAR10.
\subsection{Computing Infrastructure} \label{app:computing}
All experiments are performed on servers equipped with Intel(R) Xeon(R) Gold 6226R CPUs (2.90GHz, 520GB RAM) and NVIDIA RTX A6000 GPUs (8 units), running Ubuntu 20.04.
Our implementation is based on CUDA 11.7, PyTorch 2.0.1 for ImageNet, TensorFlow/Keras 2.11.0 for CIFAR10, CIFAR100, and CIFAR10-DVS.

\subsection{Experimental setup} \label{app:experimental_setup}
The input size of the model is set to 32x32 for CIFAR10/100, 224x224 for ImageNet, and 48x48 for CIFAR10-DVS.
On CIFAR10/100, we trained each model for 310 epochs with the AdamW optimizer and a cosine decay learning rate scheduler with a 20-epoch warm-up.

\begin{wraptable}{r}{8cm}
\centering
\caption{Distribution of spikes in the Conv1 layer for \textit{BOX}, \textit{TRI}, and \textit{ASY}. Each value represents the percentage of samples with the corresponding spike counts.}
\label{tab:spike_counts}
\renewcommand{\arraystretch}{0.7}
\begin{tabular}{c|ccc}
\toprule
\textbf{Spike Counts} & \textit{\textbf{BOX}} & \textit{\textbf{TRI}} & \textit{\textbf{ASY}} \\
\midrule
0 & 91.73\% & 92.41\% & 93.24\% \\
1 & 6.35\%  & 5.58\%  & 5.15\%  \\
2 & 1.49\%  & 1.53\%  & 1.25\%  \\
3 & 0.29\%  & 0.32\%  & 0.24\%  \\
4 & 0.14\%  & 0.16\%  & 0.12\%  \\\midrule
\textbf{Num of Spikes} & 9.12k & 9.76k & 7.92k\\
\bottomrule
\end{tabular}
\end{wraptable}
We set it to 200 epochs for CIFAR10-DVS.
All models on CIFAR10 and CIFAR10-DVS were trained with an initial learning rate of 1$\times 10^{-5}$, a learning rate of 6$\times10^{-3}$, and a weight decay of 2$\times10^{-2}$.
On CIFAR100, all models were trained with an initial learning rate of 1$\times 10^{-4}$, a learning rate of 5$\times10^{-3}$, and a weight decay of 4$\times10^{-2}$.
Data augmentation was performed using a combination of CutMix~\cite{yun2019cutmix} and RandAugment~\cite{cubuk2020randaugment}. 
RandAugment was configured with one augmentation per image, a magnitude of 1, a magnitude standard deviation of 0.4, and an application rate of 0.5.
The batch size was set to 100 for CIFAR10/100.
For ImageNet experiments, we employed the E-SpikeFormer~\citep{yao2025scaling} model.
A batch size of 360 was used for the 10M model, while a batch size of 100 was used for the 173M model. 
Both models were trained with a base learning rate of $6\times10^{-4}$, a minimum learning rate of  1$\times 10^{-6}$, and 5 warm-up epochs. 
For data augmentation, we applied Mixup~\citep{zhang2018mixup} with a weight of 0.8 and CutMix with a weight of 1.0.
For all experiments, the effective window $\beta$ was initialized to 0.5, and the optimal $\beta$ was searched within $\beta \in [0.1, 1.0]$ at the first iteration of every epoch.
For the gradient-bias term $h$ in Eq.~\ref{eq:asymmetric_sg_function}, we set it to $0.6$ for VGG16 and VGGSNN, and $0.75$ for ResNet19 and E-SpikeFormer.

\subsection{Visualization of Segmentation} \label{app:segmentation_app}
\begin{figure*}[t]
    \centering
    \includegraphics[width=\textwidth]{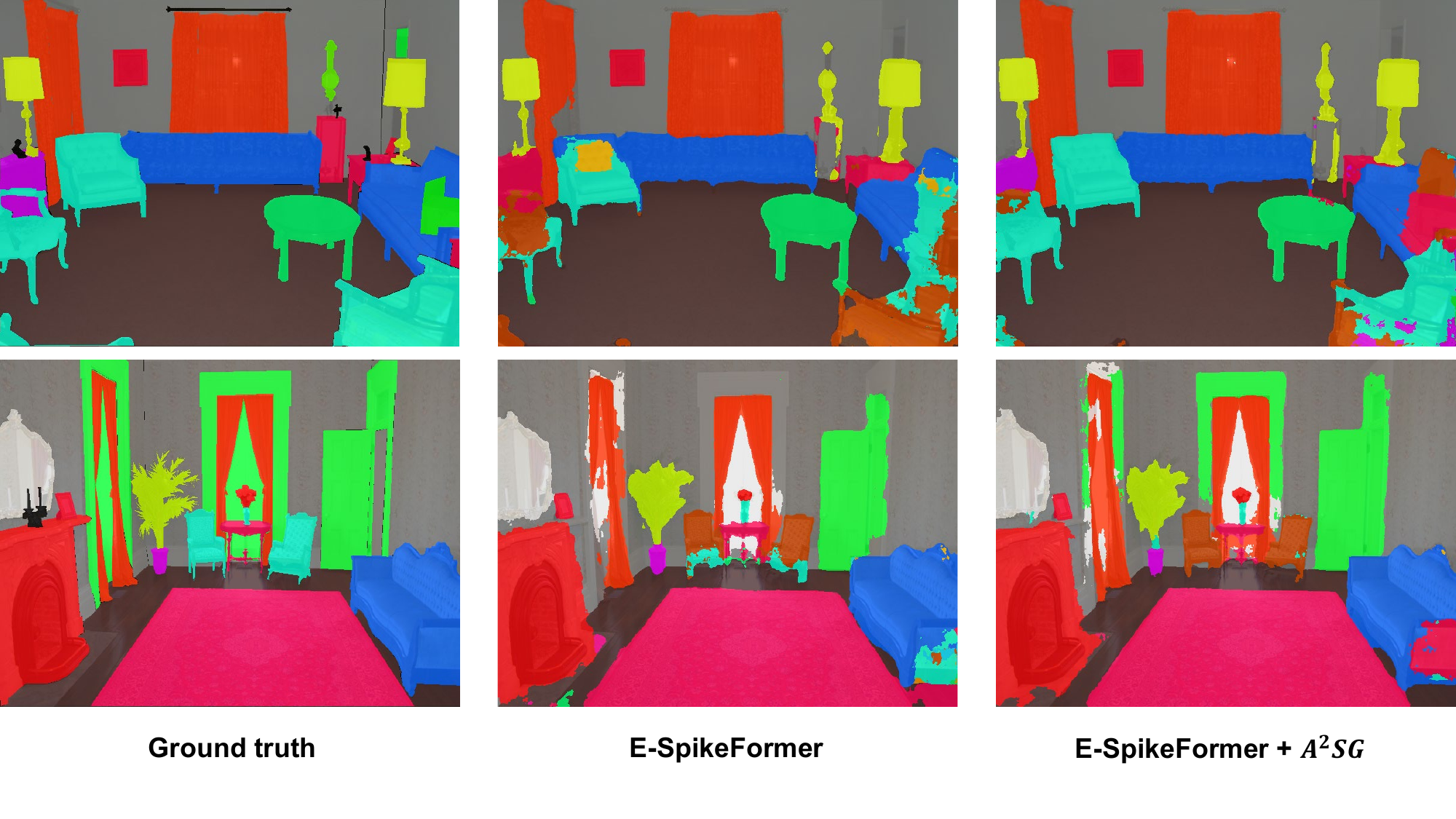}
    \caption{Segmentation results on the ADE20K dataset.}
    \label{fig:segmentation_fig}
\end{figure*}

Fig.~\ref{fig:segmentation_fig} presents segmentation results on the ADE20K dataset.
Compared to the E-SpikeFormer, the proposed $A^2SG$ improves boundary sharpness and object consistency.

\subsection{Quantitative Analysis of the Asymmetric Surrogate Gradient} \label{app:quant_anal}

We provide additional analysis of the effect of \textit{ASY} on VGG16 using CIFAR10. 
As shown in Table~\ref{tab:spike_counts}, in the Conv1 layer, \textit{ASY} substantially increases the proportion of silent neurons and reduces the total spike count (7.92k), outperforming both \textit{BOX} (9.12k) and \textit{TRI} (9.76k), while also yielding higher test accuracy. 
To further interpret this result, we analyze the membrane potential statistics. 
In SNNs, the weight distribution is mapped linearly to each neuron's membrane potential. 
Given an input spike count $M$, the membrane potential $u$ can be expressed as
\begin{equation}
u = \sum_i w_i x_i ,
\end{equation}
where $x_i$ and $w_i$ denote the presynaptic spikes and their corresponding weights, respectively.  
The mean and standard deviation of the membrane potential are
$\mu_u = M\mu_w$ and $\sigma_u=\sqrt{M}\sigma_w$,  
where $\mu_w$ and $\sigma_w$ are the mean and standard deviation of the weight distribution.  
By the central limit theorem, the membrane potential can be approximated by $\mathcal{N}(\mu_u, \sigma_u^2)$.  
To quantify the distance between the mean membrane potential and the firing threshold $V_{\mathrm{th}}$, we define the relative membrane distance (RMD) as
\begin{equation} \label{eq:RMD}
\operatorname{RMD} := \frac{V_{\mathrm{th}}-\mu_u}{\sigma_u}.
\end{equation}
A larger RMD indicates that the membrane potential lies farther below the threshold, thus reducing the probability of spiking.  

\begin{table}[h]
    \centering
    \caption{$\mu_w$, $\sigma_w$, and RMD depending on the surrogate gradient (SG) functions in Conv1.}
    \label{tab:weight_mean_std}
    \resizebox{8cm}{!}{%
        \begin{tabular}{l c c c c}
            \toprule
            SG func. & $\mu_w$ ($\times10^{-3}$) & $\sigma_w$ & RMD & \# of spikes ($\times10^3$)\\
            \midrule
            \textit{BOX} & -8.122 & 0.1408 & 7.159 & 9.12 \\
            \textit{TRI} & -8.205 & 0.1394 & 7.231 & 9.76 \\
            \textit{ASY} & -7.002 & 0.1361 & 7.398 & 7.92 \\
            \bottomrule
        \end{tabular}
    }
    \vspace{0.2em}
\end{table}

As summarized in Table~\ref{tab:weight_mean_std}, \textit{ASY} achieves the highest RMD and the lowest spike count, confirming its effectiveness in suppressing redundant spikes and improving energy efficiency.

\subsection{Comparison of SGV, TGC, and $\beta$ dynamics on Other Layers} \label{app:other_layer}

\begin{figure}[t]
    \centering
    \includegraphics[width=\linewidth]{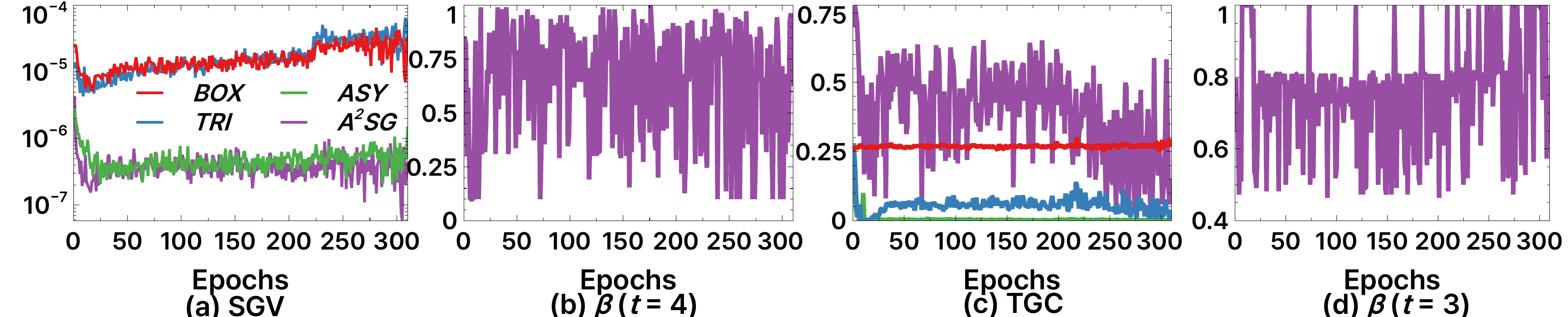}
    \vspace{-2.0em}
    \caption{Comparison of SGV, TGC, and $\beta$ dynamics across different surrogate gradient functions at Conv1 layer. (a) SGV over epochs, (b) $\beta$ dynamics at 
$t=4$, (c) TGC over epochs, (d) $\beta$ dynamics at $t=3$.}
    \label{fig:conv_1_set}
\end{figure}

\begin{figure}[t]
    \centering
    \includegraphics[width=\linewidth]{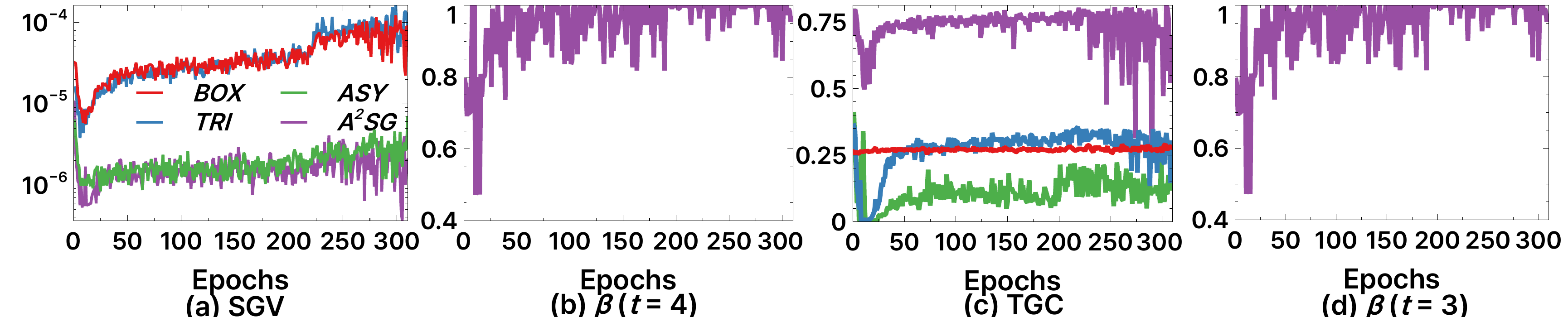}
    \vspace{-2.0em}
    \caption{Comparison of SGV, TGC, and $\beta$ dynamics across different surrogate gradient functions at Conv3 layer. (a) SGV over epochs, (b) $\beta$ dynamics at 
$t=4$, (c) TGC over epochs, (d) $\beta$ dynamics at $t=3$.}
    \label{fig:conv_3_set}
\end{figure}

Figs.~\ref{fig:conv_1_set} and ~\ref{fig:conv_3_set} illustrate SGV, TGC, and $\beta$ for Conv1 and Conv3, respectively. 
Similar to the results in Conv5-2, \textit{A$^{2}$SG} achieves the lowest SGV and the highest TGC. 
Notably, while \textit{ASY} exhibited relatively high TGC in Conv5-2, it shows much lower TGC in the earlier layers.
In contrast, \textit{A$^{2}$SG} consistently maintains high TGC across all layers by applying the adaptive surrogate strategy to \textit{ASY}. 
Furthermore, in Figs.~\ref{fig:conv_1_set} and ~\ref{fig:conv_3_set}-(b), (d), it can be observed that $\beta$ is adjusted in a manner that improves both SGV and TGC.

\subsection{Compatibility}
\begin{table}[t]
    \centering
    \caption{LIF and PLIF comparison on CIFAR10 with VGG16.}
    \label{tab:other_neuron}
    {\scriptsize
    \renewcommand{\arraystretch}{0.9}
    \setlength{\tabcolsep}{2pt}
    \resizebox{0.5\linewidth}{!}{%
    \begin{tabular}{llcc}
        \toprule
        \multirow{2}{*}{Neurons} & \multirow{2}{*}{Methods} & \multirow{2}{*}{Acc. (\%)} & \# of Spikes \\
        && & ($\times 10^3$) \\
        \midrule
        \multirow{2}{*}{LIF} & Baseline & 94.84$\pm$0.05 & 94.6$\pm$1.0 \\
        & \textbf{\textit{\boldmath{$A^2$}SG} (Ours)} & \textbf{95.29$\pm$0.04} & \textbf{84.9$\pm$1.8} \\
        \cmidrule{1-4}
        PLIF & Baseline & 94.99$\pm$0.03 & 91.6$\pm$7.4 \\
        \cite{fang2021incorporating} &
        \textbf{\textit{\boldmath{$A^2$}SG} (Ours)} & \textbf{95.33$\pm$0.01} & \textbf{82.2$\pm$0.5} \\
        \bottomrule
    \end{tabular}
    }}
\end{table}
\begin{table}[t]
    \centering
    \caption{Compatibility of $A^2SG$ with various training algorithms on CIFAR100 with VGG16.}
    \label{tab:other_methods}
    {\scriptsize
    \renewcommand{\arraystretch}{0.9}
    \setlength{\tabcolsep}{10pt}
    \resizebox{0.5\linewidth}{!}{%
    \begin{tabular}{lcc}
        \toprule
        \multirow{2}{*}{Methods} & \multirow{2}{*}{Acc. (\%)} & \# of Spikes \\
        &  & ($\times 10^3$) \\
        \midrule
        tdBN & 74.24$\pm$0.08 & 100.3$\pm$1.2 \\
        RMP~\cite{guo2023rmp}  & 74.39$\pm$0.07 & 102.8$\pm$2.9 \\
        \midrule
        \textbf{tdBN\textit{ + \boldmath{$A^2$}SG}} & \textbf{75.21$\pm$0.08} & \textbf{100.2$\pm$0.4} \\
        \textbf{RMP\textit{ + \boldmath{$A^2$}SG}}  & \textbf{75.25$\pm$0.03} & \textbf{102.3$\pm$1.1} \\
        \bottomrule
    \end{tabular}
    }}
\end{table}
\setlength{\textfloatsep}{0pt}

We evaluate whether $A^2SG$ can be seamlessly combined with diverse training methods and neuron models.
Tab.~\ref{tab:other_neuron} further confirms consistent gains for both LIF and PLIF neurons on CIFAR10 with VGG16, improving accuracy and reducing the number of spikes.
Tab.~\ref{tab:other_methods} shows that integrating $A^2SG$ with tdBN or RMP-loss improves accuracy on CIFAR100 with VGG16 while maintaining comparable spike counts.
Overall, these results indicate that $A^2SG$ is compatible with various neuron models and training methods.

\subsection{Quantification of Temporal Gradient Confusion}  \label{app:tgc_quant}
  We quantify temporal gradient confusion through TGC (Eq.~\ref{eq:TGC}), defined as the cosine similarity between gradients
   at adjacent timesteps.
  In STBP, the weight update aggregates gradient contributions across all timesteps (Eq.~\ref{eq:temporal_grad}); when TGC
  is low, gradients at different timesteps point in conflicting directions, causing partial cancellation in the aggregated
  update.
  Tab.~\ref{tab:tgc_quant} reports the average TGC across architectures and datasets.
  Across all settings, \textit{A$^2$SG} reliably improves TGC over the baseline, confirming that \textit{A$^2$SG}
  effectively mitigates temporal gradient confusion regardless of architecture and dataset.
  \begin{table}[h]
      \centering
      \caption{Average TGC across architectures and datasets.}
      \label{tab:tgc_quant}
      {\small
      \renewcommand{\arraystretch}{0.9}
      \setlength{\tabcolsep}{14pt}
      \begin{tabular}{lcc}
          \toprule
           & Baseline & \textit{A$^2$SG} \\
          \midrule
          VGG16 (CIFAR10)        & 0.272 & \textbf{0.823} \\
          ResNet19 (CIFAR10)     & 0.212 & \textbf{0.713} \\
          VGGSNN (CIFAR10-DVS)   & 0.274 & \textbf{0.814} \\
          \bottomrule
      \end{tabular}}
  \end{table}

\subsection{Noise Robustness}

\begin{table}[t]
    \centering
    \caption{Results of accuracy and spike count under different deletion rates on CIFAR10 with VGG16.\\
    }
    \label{tab:noise_robust}
    {\small  
    \renewcommand{\arraystretch}{0.8} 
    \resizebox{0.5\linewidth}{!}{
        \begin{tabular}{ccll}
            \toprule
            \multirow{2}{*}{Methods} & \multirow{2}{*}{Ratio(\%)}& \multirow{2}{*}{Acc. (\%)} & \# of Spikes\\
            &  &  & ($\times10^3$)\\
            \midrule
     	    \multirow{3}{*}{\begin{tabular}{@{}c@{}}\textit{BOX} \\ (Baseline)\end{tabular}}
            & 0 & 94.84 (-0.00) & 94 (-0.0\%) \\
     	    & 10 & 91.09 (-3.21) & 88 (-6.0\%) \\
     	    & 20 & 69.08 (-25.22) & 83 (-11.7\%) \\ 
            \cmidrule{1-4} 
     	    \multirow{3}{*}{\begin{tabular}{@{}c@{}}\textbf{\textit{\boldmath{$A^2$}SG}} \\ \textbf{(Ours)}\end{tabular}}
            & 0 & 95.29 (-0.00) & 89 (-0.0\%) \\
     	    & 10 & \textbf{92.59 (-2.81)} & \textbf{84 (-5.6\%)} \\
     	    & 20 & \textbf{69.08 (-24.89)} & \textbf{81 (-9.0\%)} \\
            \bottomrule
        \end{tabular}%
    }}
\end{table}

\begin{figure}[t]
    \centering
    \includegraphics[width=0.5\linewidth]{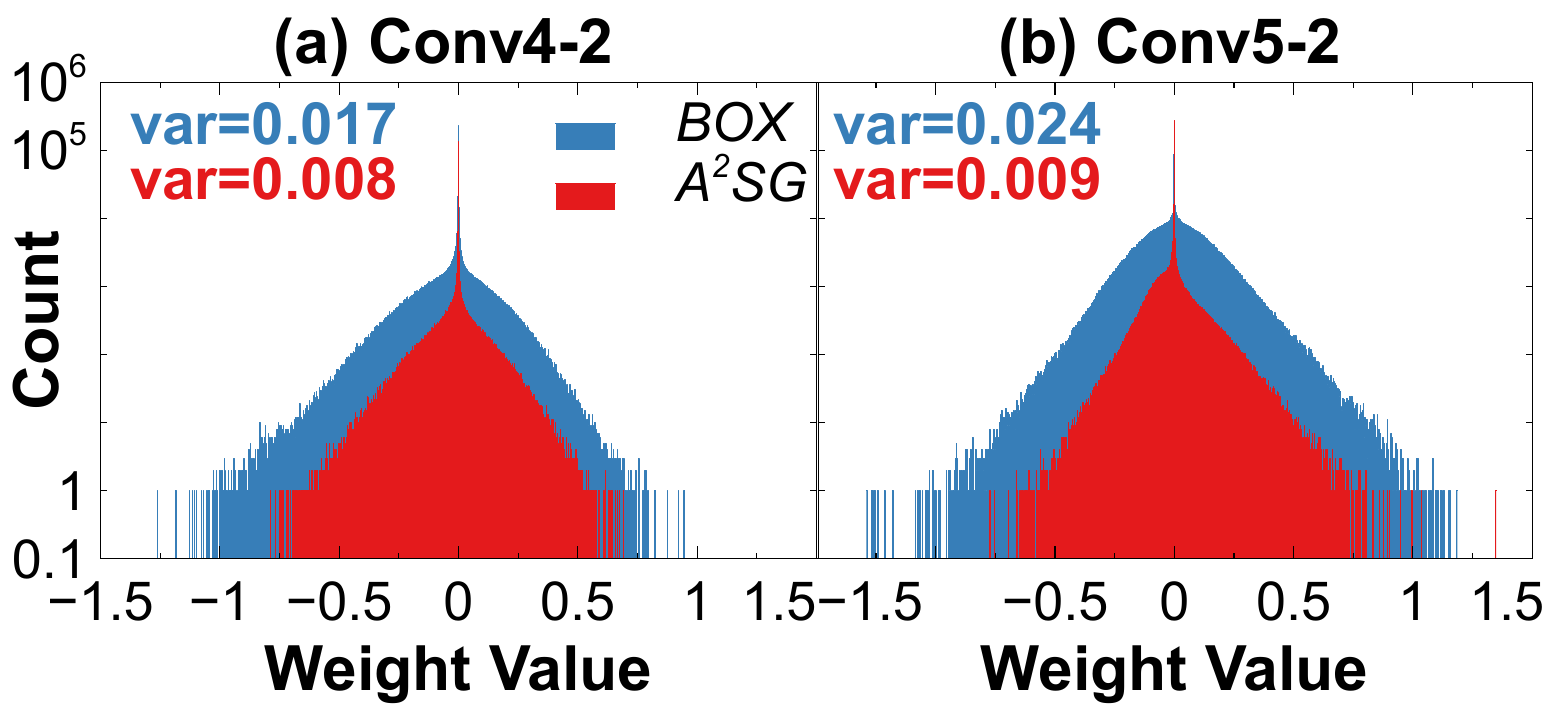}
    \caption{Weight distribution of \textit{BOX} and \textit{A$^2$SG}. (a) and (b) are Conv4-2 and Conv5-2 layers.}
    \label{fig:noise}
\end{figure}
Tab.~\ref{tab:noise_robust} reports accuracy and spike counts under deletion noise, where a fixed fraction of spikes is removed from each layer.
Overall, the results indicate that $A^2SG$ is less sensitive to missing spike events, yielding improved robustness under noisy spike observations.
In addition, Fig.~\ref{fig:noise} illustrates the weight distributions, where $A^2SG$ shows lower variance than $BOX$.
This is consistent with the improved generalization and robustness observed in the experiments.

\end{document}